\documentclass[preprint,12pt]{clear2025}

\usepackage{amsmath,amsfonts,bm}
\usepackage{pifont}

\newcommand{\indep}{\,\perp\!\!\!\!\!\!\perp\,} 

\usepackage{amssymb}
\usepackage{mathtools}
\usepackage{adjustbox}
\usepackage{rotating}

\newtheorem{hypothesis}{Hypothesis}[section]

\usepackage{mathrsfs}

\RequirePackage{algorithm}
\RequirePackage{algorithmic}

\def\eqref#1{equation~\ref{#1}}

\def\1{\bm{1}}

\def\rvb{{\mathbf{b}}}

\def\rvx{{\mathbf{x}}}

\def\rvz{{\mathbf{z}}}

\def\rmA{{\mathbf{A}}}
\def\rmB{{\mathbf{B}}}

\def\rmH{{\mathbf{H}}}
\def\rmI{{\mathbf{I}}}

\def\rmM{{\mathbf{M}}}

\def\va{{\bm{a}}}
\def\vb{{\bm{b}}}
\def\vc{{\bm{c}}}

\DeclareMathAlphabet{\mathsfit}{\encodingdefault}{\sfdefault}{m}{sl}
\SetMathAlphabet{\mathsfit}{bold}{\encodingdefault}{\sfdefault}{bx}{n}

\def\gA{{\mathcal{A}}}
\def\gB{{\mathcal{B}}}
\def\gC{{\mathcal{C}}}

\def\gN{{\mathcal{N}}}

\def\gS{{\mathcal{S}}}
\def\gT{{\mathcal{T}}}
\def\gU{{\mathcal{U}}}
\def\gV{{\mathcal{V}}}
\def\gW{{\mathcal{W}}}

\DeclareMathOperator*{\Esample}{\widehat{\mathbb{E}}}

\DeclareMathOperator*{\E}{\mathbb{E}}

\newcommand{\Var}{\mathbb{V}}

\newcommand{\pa}{\mathcal{P}}
\newcommand{\graph}{\mathscr{G}}

\usepackage{marginnote}
\usepackage{booktabs}						
\usepackage{multirow}						
\usepackage{amsfonts}						
\usepackage{graphicx}						
\usepackage{duckuments}						
\usepackage{lipsum}
\usepackage{wrapfig}
\usepackage{enumitem}
\usepackage[rightcaption]{sidecap}
\sidecaptionvpos{figure}{c}

\usepackage{lipsum}

\usepackage{caption}

\title[Generating CTIGs for Counterfactual Validation of TLP]{Generating Causal Temporal Interaction Graphs for \\Counterfactual Validation of Temporal Link Prediction}
\usepackage{times}

\clearauthor{%
 \Name{Aniq Ur Rahman} \&
 \Name{Justin P. Coon}  \Email{ \{aniq.rahman, justin.coon \}@eng.ox.ac.uk}\\
 \addr University of Oxford
}

\begin{document}

\maketitle

\begin{abstract}%
Temporal link prediction (TLP) models are commonly evaluated based on predictive accuracy, yet such evaluations do not assess whether these models capture the causal mechanisms that govern temporal interactions. In this work, we propose a framework for counterfactual validation of TLP models by generating causal temporal interaction graphs (CTIGs) with known ground-truth causal structure. We first introduce a structural equation model for continuous-time event sequences that supports both excitatory and inhibitory effects, and then extend this mechanism to temporal interaction graphs. To compare causal models, we propose a distance metric based on cross-model predictive error, and empirically validate the hypothesis that predictors trained on one causal model degrade when evaluated on sufficiently distant models. Finally, we instantiate counterfactual evaluation under (i) controlled causal shifts between generating models and (ii) timestamp shuffling as a stochastic distortion with measurable causal distance. Our framework provides a foundation for causality-aware benchmarking.
\end{abstract}



\thispagestyle{empty}

\section{Introduction}
Temporal Interaction Graphs (TIGs) are used to model real-world phenomena in which entities interact for a certain duration, which may be instantaneous. A social network serves as a prime example, where edges form between users interacting through messaging or commenting on each other's posts.

In graph learning, \textit{link prediction} refers to the task of predicting the existence of edges in a static graph after having observed other parts of the graph. In the dynamic setting, we have \textit{temporal link prediction} (TLP), wherein after observing a TIG up to a certain timestamp, a model makes  predictions on the existence of edges at specific timestamps in the future.

While the problem of TLP itself has gained popularity in the past years \citep{kumar2019predicting, rossi_temporal_2020, yu2023towards}, the question of whether these models capture the underlying causal mechanisms in the data remains largely unexplored~\citep{ur2025primer}. Recently,
\citet{rahman2025rethinking} have evaluated TLP models in a counterfactual setting, concluding that the models cannot distinguish between TIGs if the order in which the edges occur is changed, or the frequency with which they appear is changed. However, the real-world datasets are not certifiably causal \citep{pearl2009causality}; in particular, the causal influence of past temporal links on future interactions is generally unknown. To address this limitation, we propose a method for generating causal temporal interaction graphs (CTIGs), enabling counterfactual evaluation of TLP models using data with known causal structure.

Although research on causal temporal interaction graphs (CTIGs) is limited, insights can be drawn from the well-established adjacent field of \textit{causal event sequences} (CES) \citep{kim2011granger, noorbakhsh2022counterfactual, jalaldoust2022causal, cupperscausal} and from work on \textit{event prediction} \citep{zhao2021event}. Accordingly, we first develop a model to generate CESs and then naturally extend it to generate CTIGs, where each event corresponds to an edge event.
Once a CTIG is generated from a causal model, we are interested in quantifying how it differs from a CTIG generated by a different causal model, which leads to the following question:
\begin{quote}
    \textit{How do we measure the distance between two CTIG-generating causal models?}
\end{quote}
This question also lays the foundation for testing the following \textit{hypothesis}:
\begin{quote}
\textit{A predictive model trained on data generated by a causal model will perform worse on data generated by another causal model, provided that the distance between the two models is sufficiently large.}
\end{quote}

To test this hypothesis, we generate CTIGs from different causal models and query an oracle with access to the true causal parameters to make predictions. We then extend this framework to incorporate causal distortions and repeat the counterfactual experiments of \citet{rahman2025rethinking} using some well-known temporal link prediction (TLP) models.

\paragraph{Organisation}
In \textsection~\ref{sec:CES}, we introduce a structural equation model for generating causal event sequences and discuss the properties of the resulting causal model. Then, we compare our proposed model to the literature in \textsection~\ref{sec:relworks}. In \textsection~\ref{sec:compare}, we propose a distance measure to quantify the difference between two causal models. Next, in \textsection~\ref{sec:CTIG}, we extend the CES generation procedure to create causal temporal interaction graphs (CTIGs), beginning with the generation of node features and using them as the basis for defining the underlying causal graph. In \textsection~\ref{sec:CF}, we present the counterfactual setup, motivate it through an empirical study, and apply it to temporal link prediction (TLP) models. We propose a method to benchmark TLP models under controlled causal shifts and timestamp shuffling.
Finally, we conclude in \textsection~\ref{sec:conclude}. Due to page limitations, the preliminaries, extended related works, and proofs of theoretical properties of the causal model, are deferred to the appendix.

\paragraph{Notation}
We use $\emptyset$ for the empty set, $|\gA|$ for the cardinality of a collection $\gA$, and $|\gA^{(y)}|$ for the number of element $y$ occurrences in $\gA$.
We use $\mathbb{N}^+$ and $\mathbb{R}^+$ for sets of positive natural and positive real numbers, respectively.
For $n \in \mathbb{N}^+$, we let
 $[n] = \{1,2, \dots, n \}$, and 
for $a,b \in \mathbb{R}$ such that $a < b$, we let  $[a, b) = \{ c \in \mathbb{R} \mid a \leq c < b \}$. 
We use uppercase letters, for example $U$, for variables, lowercase letters, for example $u$, for values, and $\varnothing$  for the null value.
 The indicator function is represented by $\mathbb{I}\{ \cdot\} \in \{0,1\}$ the argument of which is a logical statement. 
Expectation and variance is denoted by $\E[\cdot]$ and $\Var[\cdot]$, respectively. We denote sample mean by $\Esample[\cdot]$.

\section{Causal Event Sequence}
\label{sec:CES}
Consider a set of events, $\{ \va, \vb, \vc \}$, that occur continuously over time. Each data point in this sequence can be represented as a pair, where one component identifies the event, and the other specifies the time of occurrence. For instance, given an observed event sequence: $(\va, t_1),$ $(\vc, t_2),$ $(\va, t_3), (\vb, t_4), (\vc, t_5), \dots, (\va, t_m),$ the following predictive questions can be posed: (1)~Can we predict when the next event will occur? (2)~If the timing of the next event, $t_{m+1}$, is known, can we predict which event will occur at that time? (3)~Can we estimate the probability that an event $ e \in \{ \va, \vb, \vc \} $ occurs within a specified future time interval? (4)~Can we predict the number of times an event $ e \in \{ \va, \vb, \vc \} $ will occur within a future time interval?

These questions highlight fundamental challenges in forecasting events that occur in continuous time.
Now suppose there exists a function $f$ that takes an event sequence as input and produces answers to these predictive queries. A natural question then arises:
\begin{quote}
    \textit{Is there reason to believe that the occurrence of past events influences which events will occur in the future, and when they will occur?} 
\end{quote}
If past events have no impact on future events, then providing them as input to $f$ is futile.
However, \citet{gong_active_2023} note that \textit{when data arrives sequentially, it is often assumed that temporal dependencies exist}. Therefore, we propose a causal mechanism to generate event sequences in continuous time so that the assumption holds. Moreover, we derive the following \textit{design principles} for a causal mechanism based on \citep{gong_active_2023}:
\begin{enumerate}[label=(\textit{D\arabic*})]
    \item Causal relationships take time to propagate, i.e., suppose an event $\va$ causes another event $\vb$, then, if $\va$ occurs at time $t$, then event $\vb$ might occur after some delay.
    \item to infer causal relationships among various events, one must consider both the temporal order in which they occur and the delay between them.
    \item the cause and effect events unfold on a single continuous timeline on a rolling basis.
\end{enumerate}

\paragraph{Proposed Model}
Consider a sequence of events with $n$ distinct event types that occur repeatedly over time. Each event type has an associated \textit{trigger event sequence}, which determines when it may be activated. We define $U_i(t) \in \{0,1\}$ as an indicator of whether an event of type $ i $ is triggered at time $ t $.

Given past events, an event of type $ i $ may or may not occur at time $ t $, which we denote as $ X_i(t) \in \{ 0,1 \} $. To capture historical influence, we define $ X_i'(t)\in \{0,1\}$ as the presence of event $ i $ within the time window $ (t - \bar{\tau}, t) $. Therefore, the causal model can be functionally expressed as\footnote{While the variables are written in uppercase, their values are written in lowercase.}:
\begin{align}
    x_i(t) = f\Big(  \{ \underbrace{x_j'(t) : \forall j \in \pa_i }_{\text{past events}} \}, \underbrace{ u_i(t) }_{\text{trigger}} \Big). 
    \label{sem_0}
\end{align}

For each event of type $i$, we define a \textit{structural equation model} (SEM) as follows\footnote{In contrast to \eqref{sem_0}, writing the input as $\pa_i(t)$ implicitly supplies the indicators $x_j'(t)$ for all $j \in \pa_i(t)$.}:
\begin{align}
    x_i(t) = f(\pa_i(t), u_i(t)) &= u_i(t) \cdot \hat{f}(\pa_i(t))  \nonumber \\
    &= u_i(t) \cdot \mathbb{I} \left\{ \sum_{j \in \pa_i(t)} \Theta_{i,j}  \geq 0\right\}
    = u_i(t) \cdot \mathbb{I} \left\{ \sum_{j \in \pa_i} \Theta_{i,j} \, x_j'(t) \geq 0\right\},
    \label{sem}
\end{align}
where $\pa_i(t) \subseteq \pa_i \subseteq [n]$ denotes the parents of $X_i(t)$ in a causal graph\footnote{In a causal graph, a directed edge from variable $A$ to variable $B$ indicates that $A$ is a cause of $B$.}, excluding $U_i(t)$, and $\Theta_{i,j} \in [-1,1]$ is a parameter which we refer to as the \textit{inter-event causal influence}. If $\Theta_{i,j} > 0$, we say that event of type $j$ \textit{excites} the event $i$, and if $\Theta_{i,j}<0$, we say that event $j$ \textit{inhibits} event $i$. The excitatory and inhibitory influence of one event on another was studied in \citep{kim2011granger} while modelling neural spiking activity. To avoid confusion, we refer to the set $\pa_i$ as the structural parents of $i$ defined as $\pa_i = \{ j : \Theta_{i,j} \neq 0, \forall j \in [n] \}$, and $\pa_i(t) \subseteq \pa_i$ as the active parents of $i$ at time $t$ defined as $\pa_i(t) = \{ j : x_j'(t) = 1, \forall j \in \pa_i \}$. 
\textit{With slight abuse of notation, when we say that $j$ is a parent of $i$, we mean that $X_j'(t)$ is a parent of $X_i(t)$.} Therefore, even when $i = j$, the causal graph remains acyclic, as $X_i'(t)$ and $X_i(t)$ represent distinct nodes.

\begin{SCfigure}[0.9][h!]
    \centering
    \includegraphics[width=0.5\linewidth]{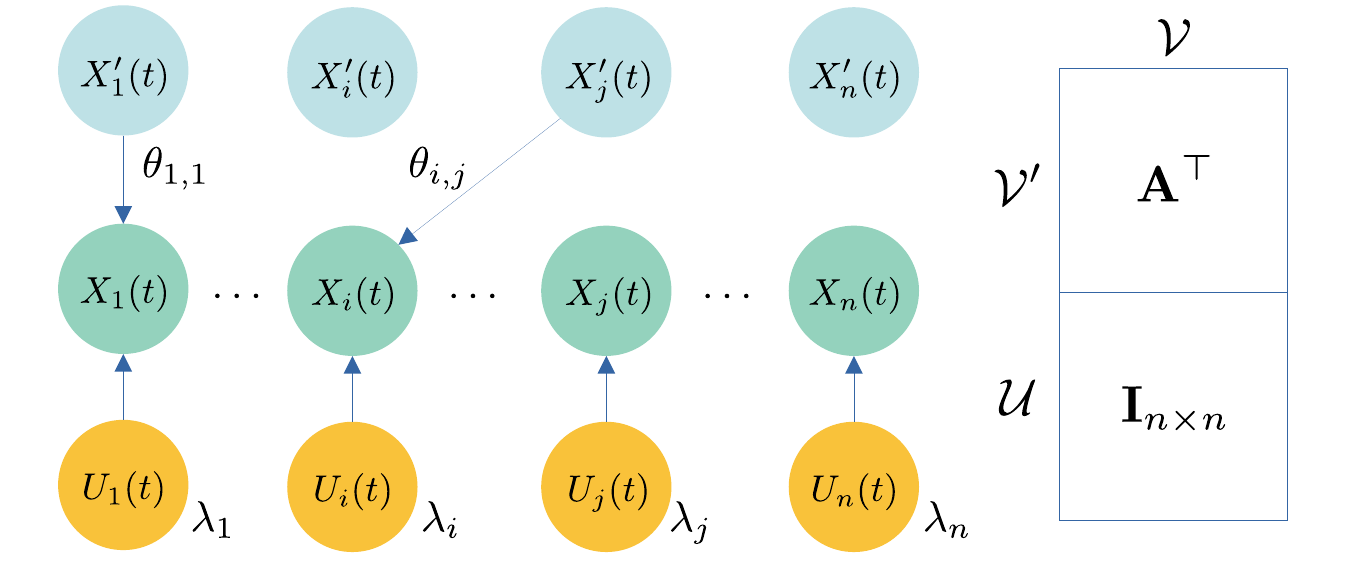}
    \caption{In this figure, we depict a sample causal graph $\graph$ with three types of nodes: event nodes $X_i$, past-event indicators $X_j'$, and trigger nodes $U_k$. The causal arrow from $X_j'(t)$ to $X_i(t)$ is annotated with the causal influence parameter $\Theta_{i,j}$. For the triggers, we indicate their generation rate $\lambda_k$, following a homogeneous Poisson process.}
    \label{fig:cgm2}
\end{SCfigure}

To specify the causal model, we describe how the causal graph, trigger event sequences, and causal influence parameters are generated:
\begin{itemize}[itemsep=2pt, topsep=2pt]
    \item For the causal graph, we focus on the bipartite graph from $\gV' = \{ X_i'(t) : i \in [n] \}$ to $\gV = \{ X_i(t) : i \in [n] \}$, whose adjacency matrix $\rmA$ is sampled from a parametric random graph model.\footnote{Standard random graph models such as the Erd\H{o}s--R\'enyi (ER) model \citep{erdos1961evolution}, the stochastic block model (SBM) \citep{abbe2018community}, or the Barab\'asi–Albert (BA) model \citep[\textsection~5]{barabasi2014network} can be used.}

    \item For each event of type $i$, its trigger event sequence $\Phi_i$ is generated by sampling a homogenous Poisson point process \citep[\textsection~2.4.1]{haenggi2012stochastic} with intensity $\lambda_i$ which can be sampled randomly from another distribution whose support is $\mathbb{R}^+$.
    \item Lastly, the \textit{inter-event causal influence} parameters $ \{ \Theta_{i,j} : \rmA_{i,j} = 1,  \forall i,j \in [n] \} $ are sampled from a distribution with support $[-1,1]$, filtering out $\{0\}$.
\end{itemize}

The procedure for generating a causal event sequence is detailed in Algorithm~\ref{algo:gen}.

\begin{algorithm}[H]
\caption{Causal Event Sequence Generation}
\label{algo:gen}
\begin{algorithmic}[1]
    \INPUT Causal model (causal graph and parameters), total time duration $T$
    \OUTPUT Event sequence $\gS$
    
    \STATE Initialize the event sequence $\gS$ as empty
    \STATE Generate trigger event sequence for each event type not exceeding $T$
    \STATE Collect and sort all trigger event times into a timeline $\gT$
    \FOR{each time $t$ in $\gT$}
        \STATE Determine which event is triggered at time $t$, say $i$
        \FOR{each event $j$ that causes $i$}
            \STATE Check if event $j$ occurred within a past time window of duration $\bar{\tau}$
        \ENDFOR
        \STATE Determine if the event $i$ occurs at time $t$ using the SEM in \eqref{sem}
        \IF{the event occurs}
            \STATE Add the event and its timestamp $(i, t)$ to the sequence $\gS$
        \ENDIF
    \ENDFOR
\end{algorithmic}
\end{algorithm}

\begin{SCfigure}[1][h!]
    \centering
    \includegraphics[width=0.45\linewidth]{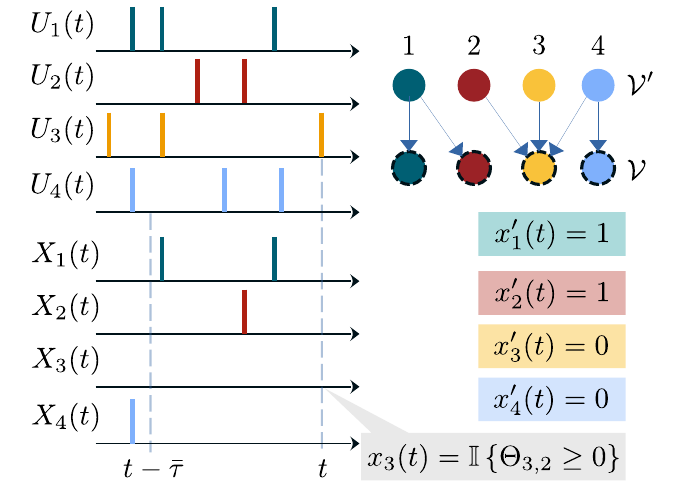}
    \caption{In this figure, we depict the working of Algorithm~\ref{algo:gen} visually for a causal graph with $4$ event nodes.
    At time $t$, $u_3(t)=1$ so we evaluate $x_3(t) = f\left(\pa_3(t), u_3(t)\right)$ and find that only $x_2'(t)=1$ from among its parents $\pa_3(t)$. This allows us to determine $x_3(t) = \mathbb{I}\left\{\Theta_{3,2} \geq 0 \right\}$. To determine the value of $x_j'(t)$ we only consider the time window $[t- \bar{\tau}, t)$. }  
    \label{fig:visual}
\end{SCfigure}

\paragraph{Properties of the proposed causal model} We now summarise several key theoretical properties of the proposed causal model, the details of which can be found in Appendix~\ref{app:properties}.
\begin{itemize}[itemsep=2pt, topsep=2pt]
    \item The causal model $( \graph, \Theta_{\graph} )$ is 
    \begin{itemize}[itemsep=2pt, topsep=2pt]
        \item \underline{semi-Markovian} for all $\rmA \in \{0,1\}^{n \times n}$, (Proposition~\ref{prop:semimarkov}).
        \item \underline{not Markovian} for all $\rmA \in \{0,1\}^{n \times n}$, (Proposition~\ref{prop:nonmarkov}).
        \item \underline{Markovian} if $\rmA = \rmI$, (Corollary~\ref{cor:markov}).
    \end{itemize}
    \item $X_i(t)$ is \underline{monotonic} relative to $X_j'(t)$ iff $\Theta_{i,j} \in [0,1]$, (Proposition~\ref{prop:mono}).
    \item $X_j'(t)$ is \underline{exogenous} relative to $X_i(t) \, \forall i,j \in [n]$,
    (Proposition~\ref{prop:exo}).
    \item If for some $j \in \pa_i$, $X_j'(t)$  exogenous relative to $X_i(t)$, and $X_i(t)$ is monotonic relative to $X_j'(t)$, then the probability of necessity and sufficiency $(\mathsf{PNS})$ is \underline{identifiable}, (Theorem~\ref{thm:identify}).
\end{itemize}

\section{Related Works}
\label{sec:relworks}
\citet{noorbakhsh2022counterfactual} introduced a method to model event sequences by selectively removing events from a point process based on a probabilistic rule. However, their approach does not explicitly define causal relationships between different event types. The authors define a multidimensional Hawkes process (MHP)~\citep{hawkes1971spectra}.

\citet{jalaldoust2022causal} propose a causal discovery algorithm to learn the Granger-causal network underlying an MHP. For evaluation, the authors also generate an event sequence by sampling an MHP, where $\lambda_i(t)$ denotes the intensity function of events in the $i^{\rm th}$ dimension. 
\begin{align}
   \textstyle  \lambda_i(t) = \bm{\mu}_i + \sum_{j \in \pa_i(t)} \int_{-\infty}^{t} \phi_{ij} (t - \tau) \, dU_j(\tau).
\end{align}
Here, $U_j$ denotes the counting process for the $j^{\rm th}$ dimension, and $U_j(t)$ corresponds to the number of events in the $j^{\rm th}$ dimension that have occurred before time $t$. Moreover, $\bm{\mu}_i$ is the underlying Poisson intensity, and $\phi_{ij}$ is the kernel function causally linking dimension $i$ to dimension $j$, i.e., $\pa_i(t) = \{ j : \phi_{ij} \neq 0 \}$. 

\citet{cupperscausal} introduce a causal model to generate event sequences wherein individual events of the cause \textit{trigger} events of the effect with dynamic delays. The causal graph $\graph$ consists of a set of source nodes $\gU$, and effect nodes $\gV$ with $\gU \cap \gV = \emptyset$. The model follows the assumption that an event of type $i \in \gU$ (cause) occurring at time $t$ \textit{causes} an event of type $j \in \gV$ (effect) at some time $t' \geq t$ with probability $\alpha_{i,j}$. It is further assumed that the events of type $i \in \gU$ are sampled from a homogenous Poisson process  with intensity $\lambda_i$, i.e., $\gS^i = \Phi(\lambda_i)$.
Some events of type $j \in \gV$ occur due to noise, described by another Poisson process with intensity $\lambda_j$, i.e., $\gW_j = \Phi(\lambda_j)$.

The events in $\gS^i$ result in delayed events in $\gS^{i,j}$ which is constructed\footnote{We have reframed \citep[eq.~2]{cupperscausal} and its description as \eqref{eq:summary}.} as follows\footnote{For practical purposes, $\xi$ is a very large finite value which works as a proxy for $\infty$.}:
\begin{align}
    \gS^{i,j} = \big\{ v \cdot (t + d) + (1-v)\cdot \xi : t \in \gS^i, d \sim   {\rm Exp}(\Theta_{i,j}), v \sim {\rm Bern}(\alpha_{i,j}) \big\} \setminus \{ \xi \}.
    \label{eq:summary}
\end{align}
Finally, all the events of type $j \in \gV$ is defined as the set $\gS^j$ (see \citep[eq.~3]{cupperscausal}):
\begin{align}
    \gS^j = \left( \bigcup_{i \in \pa_j} \gS^{i,j} \right) \cup \gW_j,
    \label{eq:fauxsem}
\end{align}
where different events in $\pa_j$ work independently to cause events of type $j \in \gV$, i.e., the structural equation groups the parent variables through a simple ``or'' operation.
Most prior works assume that causes only increase the likelihood of an effect occurring. However, real-world systems often include inhibitory effects, where an event can reduce the likelihood of another event happening. For instance, \citet{kim2011granger} explored cases where one event could either excite or suppress another. Our model incorporates both positive and negative influences, making it more flexible for diverse applications.  

Table~\ref{tab:limitations} highlights key limitations in existing literature and explains how our proposed causal model resolves them. Due to the space limitation, we have deferred it to Appendix~\ref{app:relworks}.

\section{Comparing Causal Models}
\label{sec:compare}

Having discussed how to generate a causal model and from it, an event sequence, we now turn to quantify the distance between two causal models. Specifically, we evaluate how a model performs on  CESs generated from another model, and vice versa. Finally, we report the average error across multiple CES realisations from both models. 

We denote a causal model by $\mathfrak{C}$, parameterised by the triple $(\Lambda, \Theta, \bar{\tau})$. Given a time horizon $T \in \mathbb{R}^+$, the model generates a random trigger event sequence $\Phi \in \mathbb{S}$ and an event sequence $\gS \in \mathbb{S}$ up to time $T$, which we write as $(\Phi, \gS) \sim \mathfrak{C}(T),$ with $\gS \subseteq \Phi.$

To facilitate comparisons between models, we introduce a consistent notation to distinguish their parameters and generated sequences. 
Consider two causal models, $\mathfrak{C}_A$ and $\mathfrak{C}_B$, with
$
(\Phi^A, \gS^A) \sim \mathfrak{C}_A(T),$ and $ (\Phi^B, \gS^B) \sim \mathfrak{C}_B(T).$
We then define the function $\tilde{f}_B(i,t,\gS^A,\Phi^A)$, which uses model $\mathfrak{C}_B$ to make predictions on the event sequence $\gS^A$, as\footnote{The superscript $A$ in $u_i^A(t)$ and $x_j'^{A}(t)$ indicates that these quantities are derived from $\Phi_i^A$ and $\gS^A$, respectively.}
\begin{align}
    \tilde{f}_{B}(i,t,\gS^A,\Phi^A)
    = u_i^A(t) \cdot \mathbb{I}\left\{
        \sum_{j \in \pa_i^B} \Theta_{i,j}^B \, x_j'^{A}(t) \geq 0
    \right\}.
\end{align}

We define the following distance to quantify the predictive error of $\mathfrak{C}_B$ on $\gS^A$:
\begin{align}
  d_{B}(\gS^A, \Phi^A) = \frac{1}{n} \sum_{i \in [n]} \E_{t \in \Phi_i^A} \left[   \left| \tilde{f}_{B}(i,t,\gS^A, \Phi^A) - \tilde{f}_{A}(i,t,\gS^A, \Phi^A) \right| \right].
\end{align}
To make the distance measure between the causal models $\mathfrak{C}_A$ and  $\mathfrak{C}_B$ symmetric, we define
\begin{align}
    d_{A,B}(\gS^A, \Phi^A, \gS^B, \Phi^B) = \sqrt{d_B(\gS^A, \Phi^A) \cdot d_A(\gS^B, \Phi^B)}.
\end{align}

In order to make the distance metric invariant of the sampled sequences, we calculate the mean over multiple samples, defining the final distance metric as
\begin{align}
    \bar{d}_{A,B} = \E_{\substack{ (\Phi^A, \gS^A) \sim \mathfrak{C}_A(T) \\ ( \Phi^B, \gS^B) \sim \mathfrak{C}_B(T)}}\Big[ d_{A,B}(\gS^A, \Phi^A, \gS^B, \Phi^B) \Big].
    \label{eq:distance}
\end{align}

To empirically assess the statistical stability of the proposed distance measure $\bar{d}_{A,B}$, we examine the decay of its variance as the effective sample size increases. Fig.~\ref{fig:var_decay}  provides empirical evidence that the estimator becomes more stable as more data are observed.



\begin{SCfigure}[1.3][h!]
    \centering
    \includegraphics[width=0.35\linewidth]{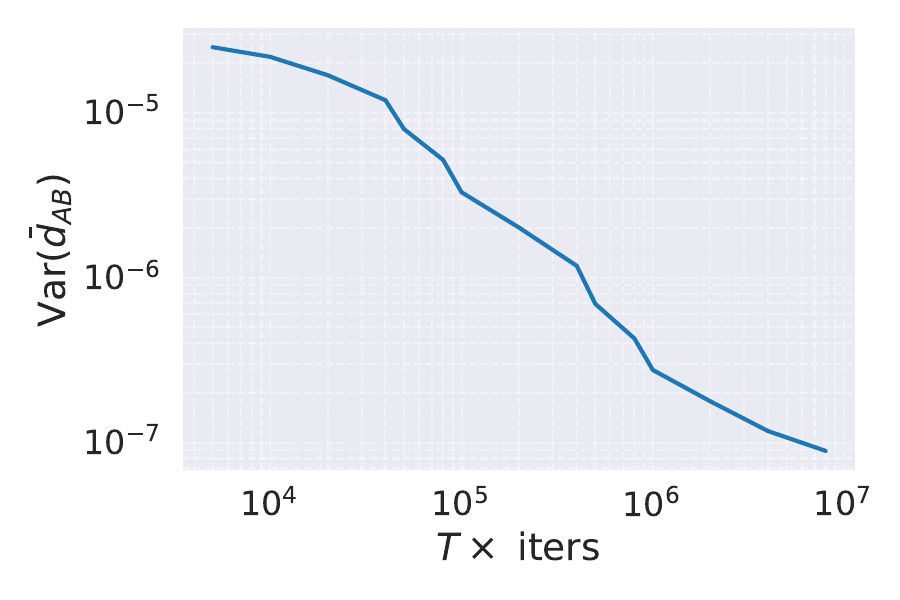}
    \caption{Empirical variance of the estimated distance $\bar{d}_{AB}$ between two independent causal models, plotted as a function of the effective sample size $T \times \text{iters}$.}
    \label{fig:var_decay}
    \vspace{-20pt}
\end{SCfigure}

\section{Causal Temporal Interaction Graphs}
\label{sec:CTIG}
In this section, we consider events as interactions between nodes, which together form a temporal interaction graph. Accordingly, instead of $n$ distinct events, we now have a graph with $n$ nodes, leading to a total of $E = \binom{n}{2}$ possible edges. 

We begin by generating node feature vectors $\rvx_a \in \mathbb{R}^r, \forall a \in [n]$, which in practice are obtained by randomly sampling from a distribution with support on $\mathbb{R}^r$, such as a multivariate normal distribution.
We then encode the relationship between two nodes $a$ and $b$ through  edge feature vector $\rvz_{(a,b)} \in \mathbb{R}^r$, computed from their respective node features $\rvx_a$ and $\rvx_b$ via a \textit{symmetric}\footnote{We assume undirected interactions; directed interactions would require $g(\cdot, \cdot)$ to be asymmetric.} function $\rvz_{(a,b)} = g(\rvx_a, \rvx_b)$ ensuring $\rvz_{(a,b)} = \rvz_{(b,a)}$. As a simple choice, we define 
\begin{align}
    g(\rvx_a, \rvx_b) = \rvx_a \odot \rvx_b.
\end{align}

Having defined edge features, we now model causal influence between edges as a function of these features.
Because causal influence between edge events is inherently directional (e.g., one interaction can influence another without reciprocity), we require this function, denoted by $h$ to be \textit{asymmetric}.
Let the features of two edges be $\rvz$ and $\rvz'$, then
\begin{align}
    h(\rvz, \rvz') = \sin\left( \nu_0 \, \tanh{ \left( \rvz^\top \rmB \, \rvz' \right) }   \right), \qquad \nu_0 \in \mathbb{R}, \quad \rmB \neq \rmB^\top.
\end{align}

The sine function bounds the output to $[-1,1]$, and $\nu_0$ adds a tunable scaling that increases the flexibility of $h$. To ensure asymmetry of $h$, we construct $\rmB$ as a skew-symmetric matrix, as follows: 
\begin{align}
    \hat{\rmB} = \begin{bmatrix}
        \hat{\rvb}_{i}
    \end{bmatrix}_{i=1}^{r}, \, \hat{\rvb}_{i} \overset{\text{\tiny i.i.d}}{\sim} \gN(\bm{0},\rmI_{r \times r}); \quad\rmB = \hat{\rmB} - \hat{\rmB}^\top.
\end{align}

The causal influence of edge $j$ on $i$ can be defined directly as $\Theta_{i,j} = h\left( \rvz_{i}, \rvz_{j} \right).$ However, this formulation is restrictive from a design perspective.
To further generalise the model, we introduce thresholding of the output of $h$ to control the \textit{sparsity} of the resulting causal graph. Moreover, we allow certain edges to be designated as \textit{non-causal}, meaning they neither exert causal influence on other edges nor are influenced by them, thereby behaving as spurious noise in the data. 

We first apply thresholding  to the output of the influence function $h$ in order to control sparsity. Specifically, for a threshold $\nu_1 \in (0,1)$, we define the function $\psi(x) = x \cdot \mathbb{I}\left\{ |x| \geq \nu_1  \right\}.$
We then define the causal influence matrix $\tilde{\Theta}$, prior to introducing non-causal edges, with entries
\begin{align}
    \tilde{\Theta}_{i, j} = \psi \left( h (\rvz_i, \rvz_j) \right).
\end{align}

In real-world scenarios, some edges correspond to spurious or unrelated events that have no causal relation with other edges. To model this, we designate $l<E$ edges as non-causal by sampling a subset of $l$ distinct edges uniformly at random from $[E]$. To this end, we construct a symmetric binary mask $\rmM_l \in \{0,1\}^{E \times E}$ with elements in the corresponding $l$ rows and columns set to $0$.
The causal influence matrix $\Theta$, is then obtained as:
\begin{align}
    \Theta = \tilde{\Theta} \odot \rmM_l.
\end{align}

Lastly, the causal graph $\rmA$ can be obtained through $\rmA_{i,j} = \mathbb{I}\left\{ \Theta_{i,j} \neq 0 \right\}.$

In Fig.~\ref{fig:causal_mat}, we illustrate the step-by-step construction of the causal influence matrix $\Theta$ for a CTIG model, instantiated on a graph with $5$ nodes.

\vspace{-10pt}
\begin{figure}[h!]
\centering
    \subfigure[$\rmH, \nu_0=100$]{\includegraphics[width=0.24\columnwidth]{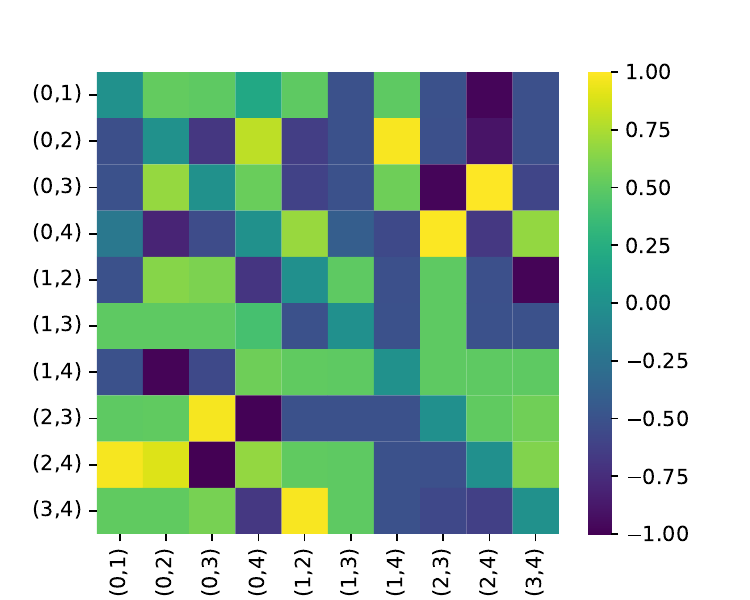}\label{fig:cmH}}
    \subfigure[ $\Tilde{\Theta}, \nu_1 = 0.55$]{\includegraphics[width=0.24\columnwidth]{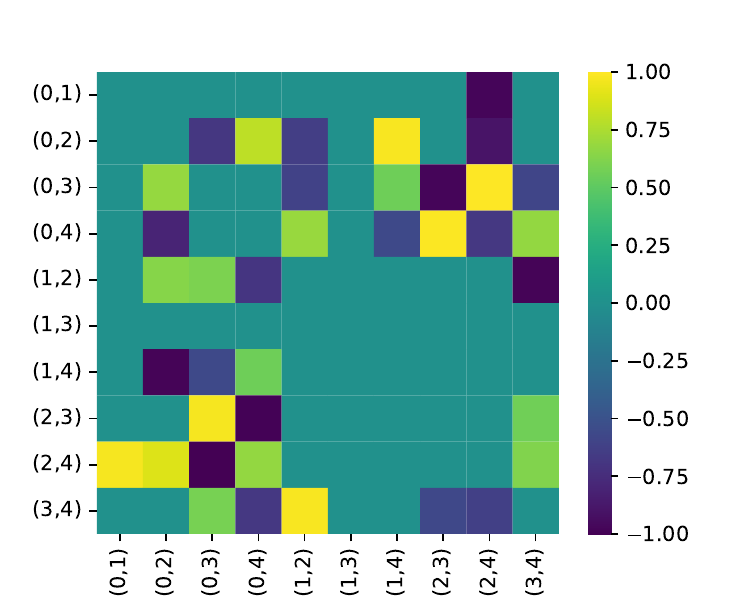}\label{fig:cmtT}} 
    \subfigure[ $\rmM_l, l=2$]{\includegraphics[width=0.24\columnwidth]{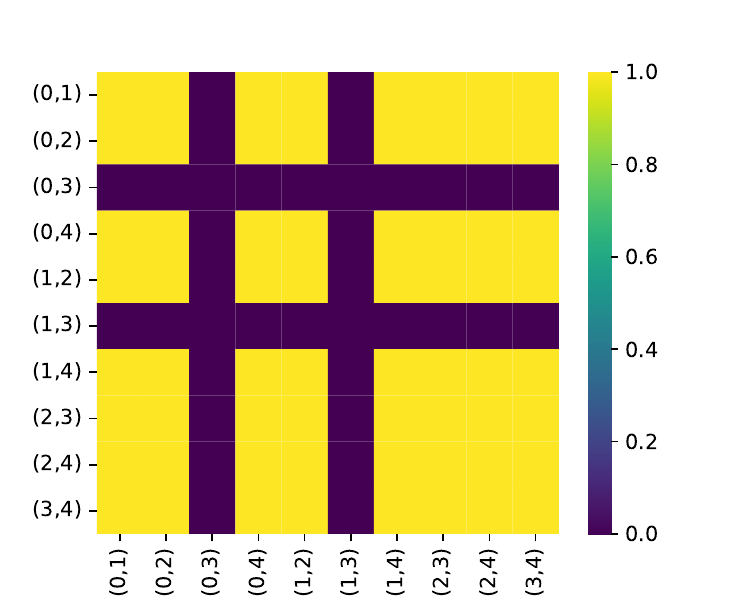}\label{fig:cmMl}} 
    \subfigure[ $\Theta$]{\includegraphics[width=0.24\columnwidth]{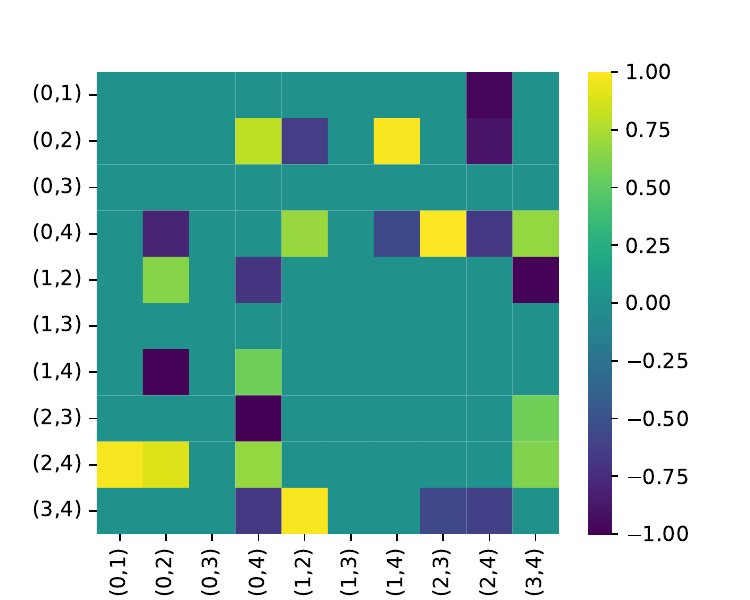}\label{fig:cmT}} 
    \caption{Construction of the CTIG causal parameters for a graph with $n=5$  and  $r=5$. }
    \label{fig:causal_mat}
    \vspace{-15pt}
\end{figure}

\section{Counterfactual Analysis}
\label{sec:CF}

\paragraph{Motivation}
Borrowing notation from \textsection~\ref{sec:compare}, we denote the
\textit{true} causal model by
$\mathfrak{C}_{0} = ( \Lambda_{0}, \Theta_{0} )$,
the \textit{estimated} model by
$\mathfrak{C}_{\star} = ( \Lambda_{\star}, \Theta_{\star} )$,
and the \textit{distorted} causal model by
$\mathfrak{C}_{\dagger} = ( \Lambda_{\dagger}, \Theta_{\dagger} )$.

The distance function of the estimated model is $d_{\star} : \mathbb{S} \times \mathbb{S} \rightarrow [0,1],$
which maps an event sequence $\gS \in \mathbb{S}$ and its trigger sequence $\Phi \in \mathbb{S}$ generated by a causal model $\mathfrak{C}$
to a scalar performance error.  Accordingly, the performance of the estimated model on a sequence generated by the true causal model is $1 -  d_\star(\gS^0, \Phi^0),$ while its performance on a sequence generated by the distorted model is $1 -  d_\star(\gS^\dagger, \Phi^\dagger).$

Comparing the performance achieved by the estimated model on the distorted sequence with the true sequence, we measure the gap $\Delta_{\star}(\gS^0, \Phi^0, \gS^\dagger, \Phi^\dagger)$ as
\begin{align}
   \underbrace{\Delta_{\star}(\gS^0, \Phi^0, \gS^\dagger, \Phi^\dagger)}_{\Delta_\star^{0,\dagger}} = \big( 1 -  d_\star(\gS^0, \Phi^0) \big) - \big( 1 -  d_\star(\gS^\dagger, \Phi^\dagger)  \big) \nonumber = \underbrace{d_\star(\gS^\dagger, \Phi^\dagger)}_{d_\star^\dagger} -\underbrace{  d_\star(\gS^0, \Phi^0)}_{d_\star^0}.
\end{align}

 We then define the mapping $\Delta_{\star}(\gS^0, \Phi^0, \gS^\dagger, \Phi^\dagger) \mapsto \bar{d}_{0,\dagger}, $ for realisations $(\Phi^0, \gS^0) \sim \mathfrak{C}_0(T),$ and  $ (\Phi^\dagger, \gS^\dagger)  \sim \mathfrak{C}_\dagger(T),$ where $\bar{d}_{0,\dagger} \in [0,1]$ is the distance between the true and distorted causal models (see \eqref{eq:distance}). This mapping can be used to analyse the correspondence between the performance gap and the distance between the causal models $\mathfrak{C}_{0}$ and $\mathfrak{C}_{\dagger}$. 

\begin{hypothesis}
\label{hypo:discrepancy}
For any causal model estimate $\mathfrak{C}_\star$ satisfying $d_\star^0 \in [0 , \delta_*)$,
there exists $\beta \in (0,1)$ such that
for all distorted models $\mathfrak{C}_{\dagger}$ with    $\bar{d}_{0,\dagger} > \beta$,
the performance gap satisfies $\Delta_{\star}^{0, \dagger} > 0.$ 
\end{hypothesis}

In other words, the estimated model performs better on the true sequence compared to the distorted sequence, if the distance of the distorted model $\mathfrak{C}_\dagger$ from the true causal model $\mathfrak{C}_0$ is large enough, and if $\mathfrak{C}_{\star}$ is sufficiently accurate, i.e., its error on the true sequence satisfies $0 \le d_\star^0 < \delta_*$. An empirical visualisation of this discussion in presented in Fig.~\ref{fig:cfPHI}.

\begin{figure}[h!]
\centering
    \subfigure[\footnotesize{$\Big( \bar{d}_{0,\dagger}, d_\star^\dagger - d_\star^0 \Big)$}]{\includegraphics[width=0.32\columnwidth]{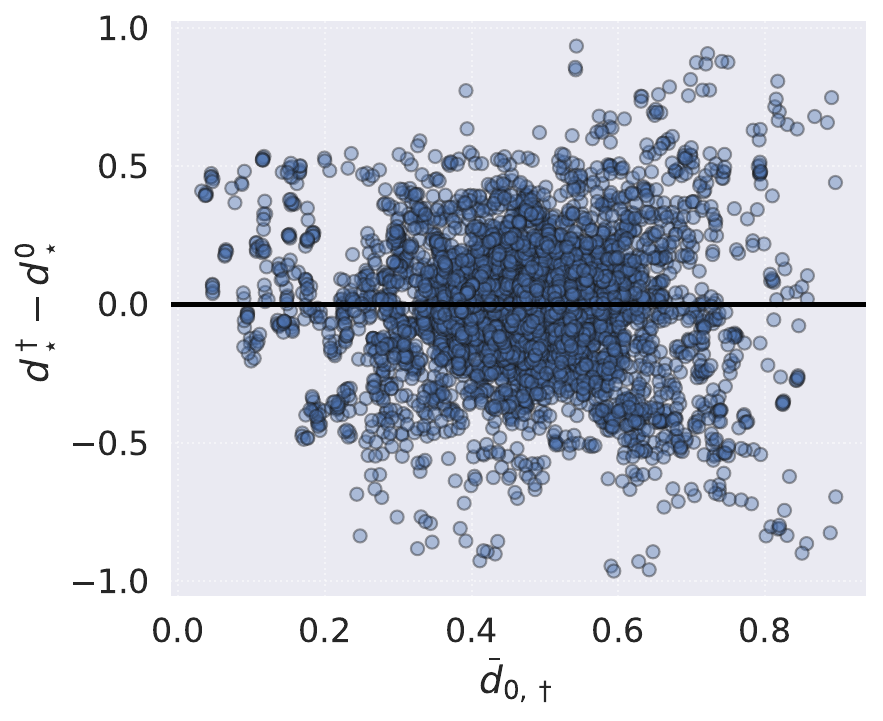}\label{fig:cfphi_1}}
    \subfigure[\footnotesize{$\Big( \bar{d}_{0,\dagger}, d_\star^\dagger - d_\star^0 \Big) : d_\star^0 < \delta_*$}]{\includegraphics[width=0.32\columnwidth]{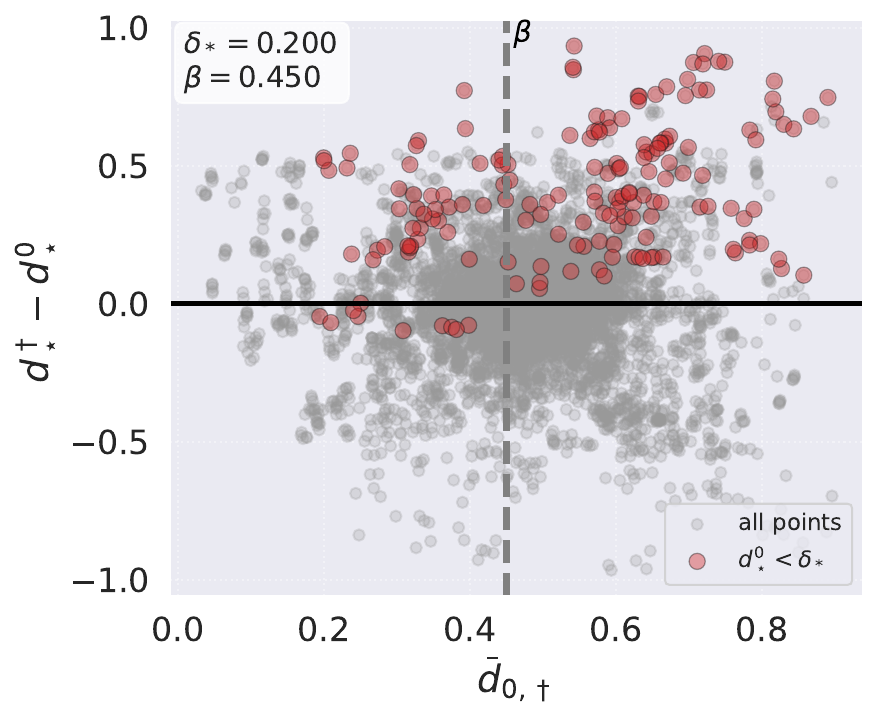}\label{fig:cfphi_2}} 
    \subfigure[\footnotesize{$\mathbb{P}\left( \Delta_\star^{0, \dagger} > 0 \mid d_\star^0 < 0.2, \bar{d}_{0, \dagger} > \beta \right)$}]{\includegraphics[width=0.32\columnwidth]{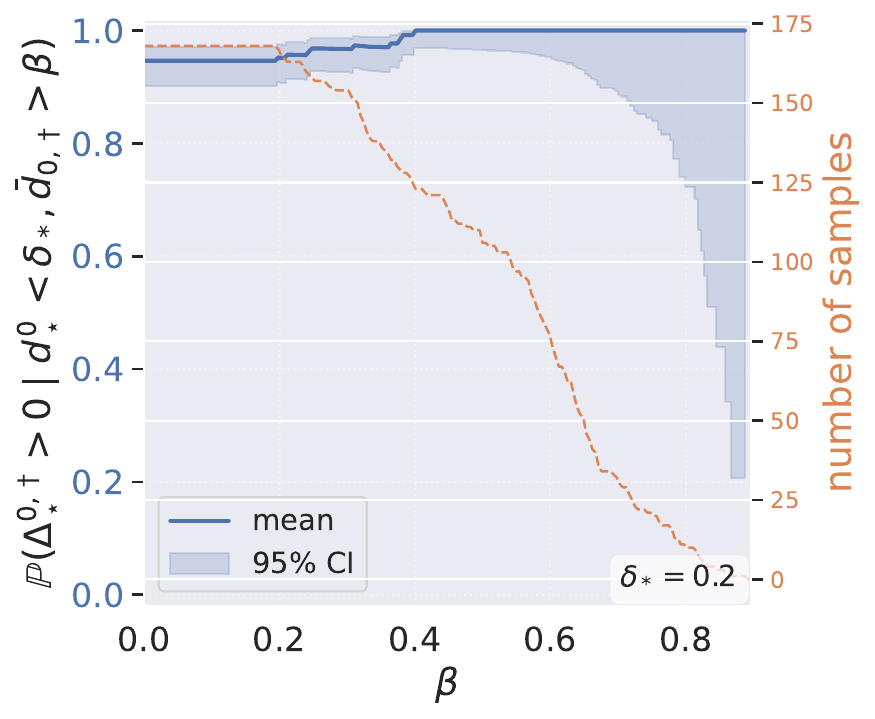}\label{fig:cfphi_3}} 
    \subfigure[\footnotesize{$\mathbb{P}\left( \Delta_\star^{0, \dagger} > 0 \mid d_\star^0 < \delta_* \right)$}]{\includegraphics[width=0.32\columnwidth]{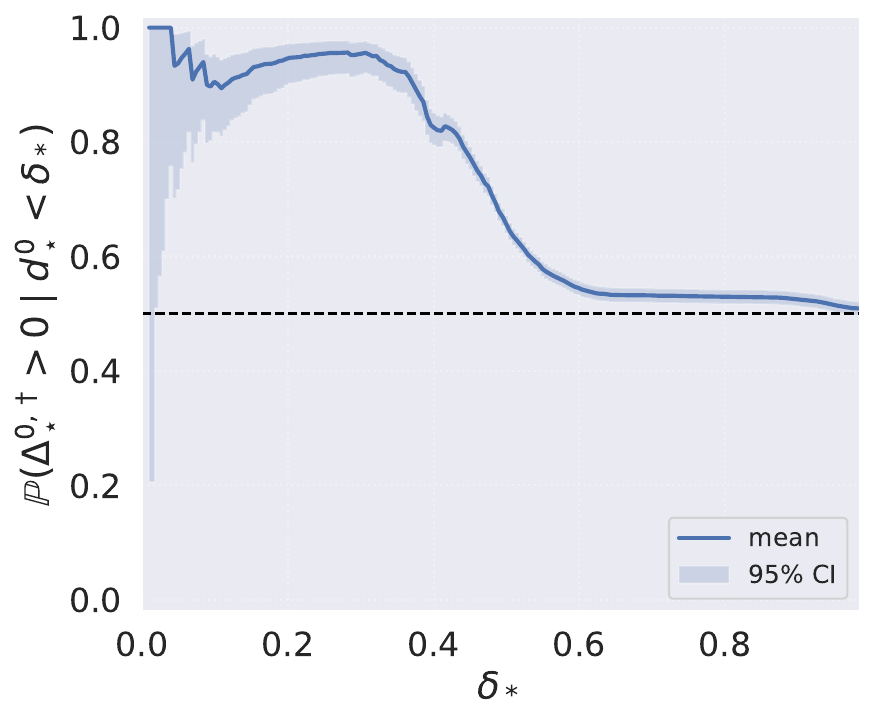}\label{fig:cfphi_4}} 
    \subfigure[\footnotesize{$\mathbb{P}\left( \Delta_\star^{0, \dagger} > 0 \mid d_\star^0 < \delta_*, \bar{d}_{0, \dagger} > \beta \right)$}]{\includegraphics[width=0.32\columnwidth]{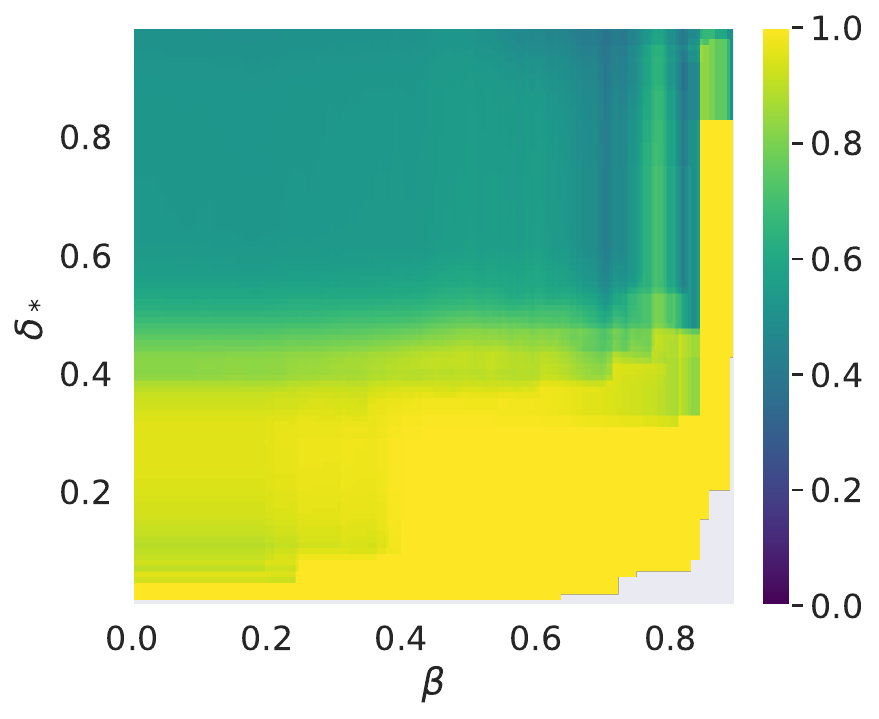}\label{fig:cfphi_5}} 
    \caption{Empirical evaluation of the performance gap under causal distortion. The panels visualize how the performance gap $\Delta_\star^{0, \dagger}$ and its associated probabilities vary with the causal distance $\bar{d}_{0, \dagger}$ and the accuracy threshold $\delta_*$, providing support for Hypothesis~\ref{hypo:discrepancy}. $n=7, T=10^3$.}
    \label{fig:cfPHI}
\end{figure}

In Fig.~\ref{fig:cfPHI}, we empirically examine the relationship between the performance gap
$\Delta_{\star}^{0,\dagger} = d_{\star}^{\dagger} - d_{\star}^{0}$
and the distance between the true and distorted causal models $\bar{d}_{0,\dagger}$.
Fig.~\ref{fig:cfphi_1} shows a scatter plot of the performance gap against $\bar{d}_{0,\dagger}$.
In Fig.~\ref{fig:cfphi_2}, we restrict attention to settings where the estimated model is sufficiently accurate, i.e., $d_\star^0 < \delta_*$ with $\delta_* = 0.2$, and highlight that for distortions satisfying $\bar{d}_{0,\dagger} > 0.45$, the performance gap is strictly positive.
Building on this observation, Fig.~\ref{fig:cfphi_3} sweeps over values of $\beta$ in the condition $\bar{d}_{0,\dagger} > \beta$ and shows that the probability of a positive performance gap increases with $\beta$.
In Fig.~\ref{fig:cfphi_4}, we instead vary $\delta_*$ and observe that the probability $\mathbb{P}(\Delta_\star^{0,\dagger} > 0 \mid d_\star^0 < \delta_*)$ decreases from near $1$ toward $0.5$ as the accuracy constraint is relaxed.
Finally, Fig.~\ref{fig:cfphi_5} presents a heat map over $(\beta, \delta_*)$, illustrating that for larger $\delta_*$, increasing $\beta$ alone is insufficient to guarantee an all-positive performance gap.

\paragraph{Application} 
Hypothesis~\ref{hypo:discrepancy} states that if the estimator
$\mathfrak{C}_{\star}$ is sufficiently accurate, i.e., $d_\star^0 \le \delta_*$, then there exists $\beta \in (0,1)$ such
that for every distorted causal model $\mathfrak{C}_{\dagger}$ with $\bar{d}_{0,\dagger} > \beta$, the corresponding performance gap satisfies $\Delta_{\star}^{0, \dagger} > 0$ with probability~$1$ under the joint sampling $(\gS^{0}, \Phi^{0}) \sim \mathfrak{C}_{0}(T)$ and $(\gS^{\dagger}, \Phi^{\dagger}) \sim \mathfrak{C}_{\dagger}(T)$. 
Formally, this can be written as
\begin{align}
\mathbb{P}\left(
    \Delta^{0, \dagger}_{\star} > 0 \mid  d_\star^0 < \delta_*, \bar{d}_{0,\dagger} > \beta
\right) = 1.
\label{eq:logic1}
\end{align}

In many practical settings, computing $\bar{d}_{0,\dagger}$ may be infeasible. For instance, one may only observe event sequences, while the parameters of the causal models that generate them remain unknown. Moreover, the true causal model $\mathfrak{C}_0$ may not be available, or the causally distorted sequences $\gS^\dagger$ may be produced by an alternative mechanism, such as corruption of the true sequence $\gS^0$ via an external procedure, rendering the direct computation of $d_\dagger^0$, and consequently $\bar{d}_{0,\dagger}$, impossible. Consequently, the statement in \eqref{eq:logic1} is not directly verifiable in practice.

However, we can still measure $d_\star^0$ as the prediction error of the estimated model on sequences generated by the true causal model. In the absence of $\bar{d}_{0,\dagger}$, we therefore use the probability
\begin{align*}
    p_\star^{0, \dagger}(\delta_*) \triangleq \mathbb{P} \left( \Delta^{0,\dagger}_{\star} > 0 \mid d_\star^0 < \delta_* \right)
\end{align*}
as a \textit{proxy} for quantifying the strength of the distortions in $\gS^\dagger$. Moreover, for a fixed $\delta_*$, this quantity enables comparison between predictors: given identical data, the predictor attaining a higher value of $p_\star^{0, \dagger}(\delta_*) $ is deemed to perform better.

\paragraph{Negative Sampling} 

In practice, trigger events are typically unobserved\footnote{TLP naturally fits within positive-unlabelled learning \citep{mansouri2025learning} setting.}; accordingly, we construct a set of negative events via \textit{negative sampling}.\footnote{For each $(i,t)\in \gS$, a negative event $(j,t)$ is generated with $j\in [E]\setminus \{ i \}$.}
We then define
\begin{align}
    \Phi \;=\; \gS \;\cup\; \left\{ (j,t) : (i,t)\in \gS,\; j \in [E]\setminus\{i\} \right\}.
    \label{eq:negsample}
\end{align}
Note that events in $\Phi \setminus \gS$ are not guaranteed to be infeasible for the causal model under consideration.\footnote{An event $(i,t)$ is said to be feasible for a causal model if $\hat{f}(\pa_i(t))=1$; see \eqref{sem}.} Consequently, $\Phi$ should not be interpreted as the set of true trigger events. In the \textit{transductive} setting, where the negative samples are only drawn uniformly from the observed set of edges so that $j \sim \mathrm{Unif} \left(\{k : (k,t) \in \gS \} \setminus \{ i \} \right).$

\paragraph{Experiment A}
We consider two causal event sequences: $\gS^0$ generated by a causal model $\mathfrak{C}_0(T)$, and $\gS^\dagger$ generated by a causal model $\mathfrak{C}_\dagger(T)$. We split\footnote{We denote the restriction of $\gS$ to the time interval $[a,b)$ by $\gS(a,b) = \{ (i,t) \in \gS : t \in [a,b) \}.$} the sequence $\gS^0$ chronologically into training and test sets as $\gS^{\mathrm{train}} = \gS^{0}(0,\tau_*)$ and $\gS^{\mathrm{test}}  = \gS^{0}(\tau_*, T)$ for some $\tau_* \in (0, T)$.
We then construct a counterfactual test set as $\bar{\gS}^{\mathrm{test}} = \gS^{\dagger}(\tau_*, T)$. Negative samples are generated in the transductive setting for both training and testing.

Let $Y$ denote a performance metric of a predictive model. Let $U$ represent the training condition, i.e., the dataset used for training, and let $X$ denote the test dataset. We use counterfactual-style notation and write $Y_{X = x}(U = u) = y$, which is interpreted as: \textit{a model trained under condition $u$ and evaluated on dataset $x$ achieves performance $y$}. For brevity, it is also written as $Y_x(u) = y$.

To instantiate this notation in our setting, we fix the training condition $u \equiv \gS^{\mathrm{train}}$ and consider two alternative test conditions, namely $x \equiv \gS^{\mathrm{test}}$ and $x' \equiv \bar{\gS}^{\mathrm{test}}$.
With $u$ fixed, the effect of data from another causal model is captured by the performance difference $Y_x(u) - Y_{x'}(u)$. We focus in particular on its sign, quantified by $\mathbb{P}\left( Y_x(u) - Y_{x'}(u) > 0 \right)$, which aligns with $\mathbb{P}\left( \Delta^{0,\dagger}_{\star} > 0 \right)$.

Fig.~\ref{fig:cf_causal_tlp} illustrates the empirical relationship between the causal shift induced by replacing $\gS^{\mathrm{test}}$ with $\bar{\gS}^{\mathrm{test}}$ and the resulting performance difference $Y_x(u)-Y_{x'}(u)$, aggregated over multiple independently generated model pairs and sequence realisations.\footnote{Each point corresponds to one realisation of $(\gS^0,\gS^\dagger)$ and one trained predictor; the horizontal axis represents the causal distance $\bar{d}_{0,\dagger}$, while the vertical axis denotes the performance gap $Y_x(u)-Y_{x'}(u)$.}
In Fig.~\ref{fig:cf_acc}, we report the performance gap achieved by the \texttt{Oracle} under counterfactual evaluation. Although the overall trend becomes increasingly positive as the causal distance grows, a small number of outliers persist. These deviations are expected, as negative sampling does not guarantee that all sampled negative events are infeasible, as discussed previously.

The remaining panels evaluate two TLP models\footnote{Extending the benchmark to additional TLP models is left to future work.}, \texttt{TGN}~\citep{rossi_temporal_2020} and \texttt{JODIE}~\citep{kumar2019predicting}. Fig.~\ref{fig:cf_ap_tlp} reports results using average precision (AP), while Fig.~\ref{fig:cf_auc_tlp} reports results using area under the ROC curve (AUC). The performance of \texttt{JODIE} remains largely flat and close to zero across increasing causal distances, indicating limited sensitivity to causal distortions. In contrast, \texttt{TGN} exhibits a clear positive trend, with the performance gap increasing as the causal distance grows, suggesting that it consistently performs better on the original test set than on its distorted counterpart. This behaviour supports the conclusion that \texttt{TGN} is causally sensitive under our counterfactual test, while \texttt{JODIE} does not.


\begin{figure}[h!]
\centering
    \subfigure[Accuracy]{\includegraphics[width=0.3\columnwidth]{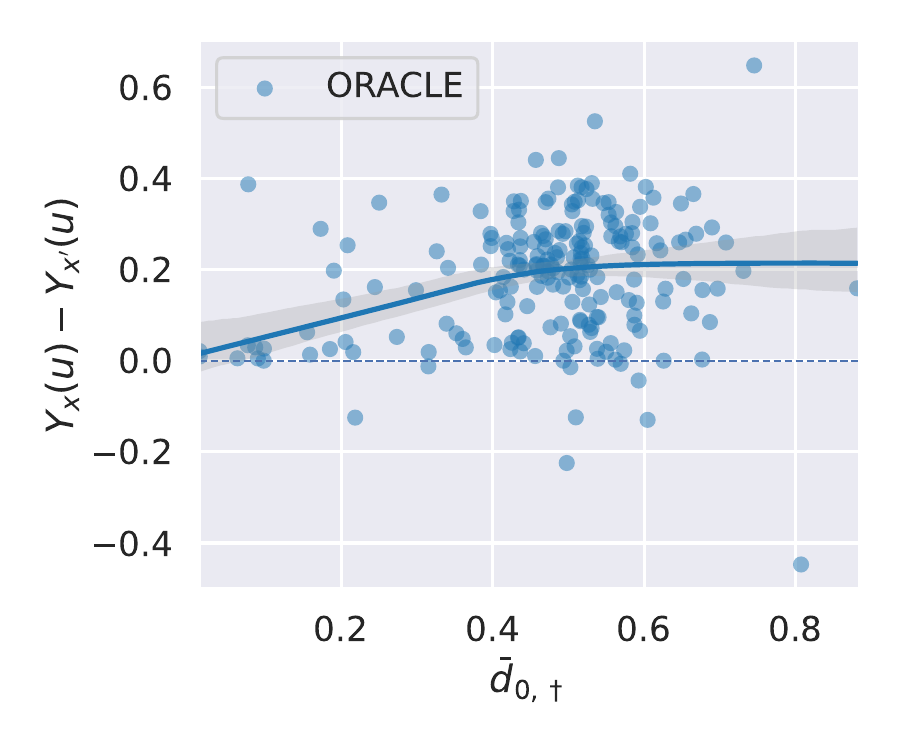}\label{fig:cf_acc}}
    \subfigure[ AP ]{\includegraphics[width=0.3\columnwidth]{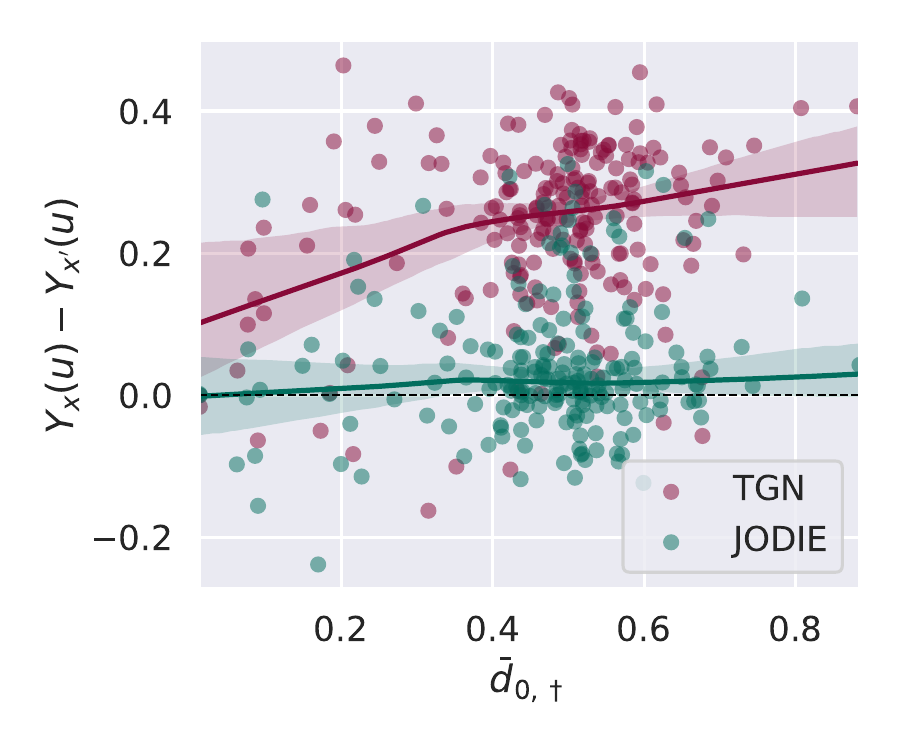}\label{fig:cf_ap_tlp}} 
    \subfigure[AUC]{\includegraphics[width=0.3\columnwidth]{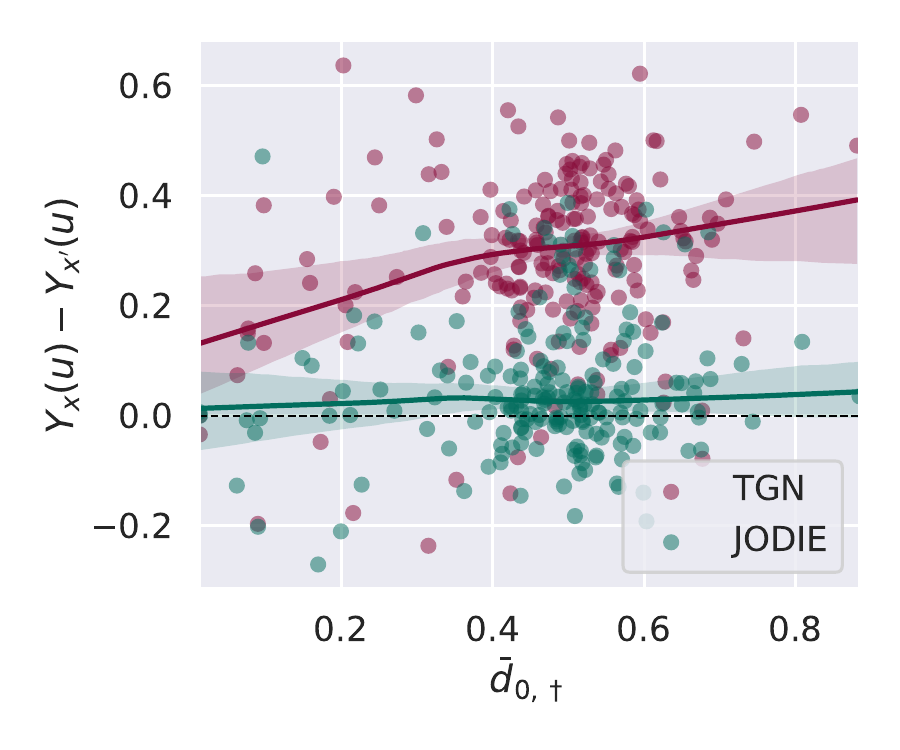}\label{fig:cf_auc_tlp}}
    \caption{Counterfactual evaluation of TLP models. Each panel shows a scatter plot of the performance gap, overlaid with a \texttt{LOWESS} smoothing curve (fraction $0.95$) to highlight the overall trend. } 
    \label{fig:cf_causal_tlp}
\end{figure}
\vspace{-10pt}

For completeness, we also report the performance $Y_x(u)$ of the models on the original test set $\gS^{\mathrm{test}}$. The \texttt{Oracle} achieves a mean accuracy of $0.693$ (95\% CI $[0.678,0.708]$), while \texttt{TGN} and \texttt{JODIE} achieve mean average precision  scores of $0.763$ (95\% CI $[0.751,0.775]$) and $0.563$ (95\% CI $[0.551,0.575]$), respectively.

\paragraph{Experiment B} We borrow the setup from the previous experiment. However, instead of constructing the counterfactual test set $\bar{\gS}^{\mathrm{test}}$ from an alternative causal model, we directly distort the original test sequence $\gS^{\mathrm{test}}$ by randomly shuffling event timestamps, thereby disrupting the temporal structure while preserving the set of observed interactions. We then repeat the same evaluation protocol as in Experiment~A, training models on $\gS^{\mathrm{train}}$ and evaluating them on both the original and shuffled versions of $\gS^{\mathrm{test}}$. 
A similar counterfactual setup based on timestamp shuffling was previously considered by \citet{rahman2025rethinking}, albeit without access to an underlying causal model. In contrast, our setting allows test sequences to be generated from a known causal model and equips us with a principled distance measure to quantify the extent to which shuffled sequences deviate from the original sequence.

Since shuffling is a stochastic procedure, we generate multiple shuffled realisations of the test data, denoted by $x_i'$ for the $i^{\text{th}}$ realisation. The causal distance between the original and shuffled sequences, as measured by the causal model $\mathfrak{C}_0$, has a mean value of $0.4322$ with a $95\%$ confidence interval of $[0.4302,\;0.4342]$. This lets us treat shuffling as a quantified causal distortion rather than an ad hoc stress test.

In Fig.~\ref{fig:shuffle_tlp}, we report violin plots of the achieved performance metrics $Y_X(u)$, where the horizontal axis corresponds to the test dataset $X$. A clear separation is observed between the metric values obtained on the original test data $x$ and those obtained on the shuffled test sets $x_1',\ldots,x_{10}'$, indicating a consistent degradation in performance under temporal distortion. A clear separation between $Y_x(u)$ and $Y_{x_i'}(u)$ is observed for the \texttt{Oracle} and \texttt{TGN} models, indicating sensitivity to temporal distortion, whereas no such separation is observed for \texttt{JODIE}. These results again support the conclusion that \texttt{TGN} passes the counterfactual test, while \texttt{JODIE} does not.
\begin{figure}[h!]
\centering
    \subfigure[Accuracy, \texttt{Oracle}]{\includegraphics[width=0.28\columnwidth]{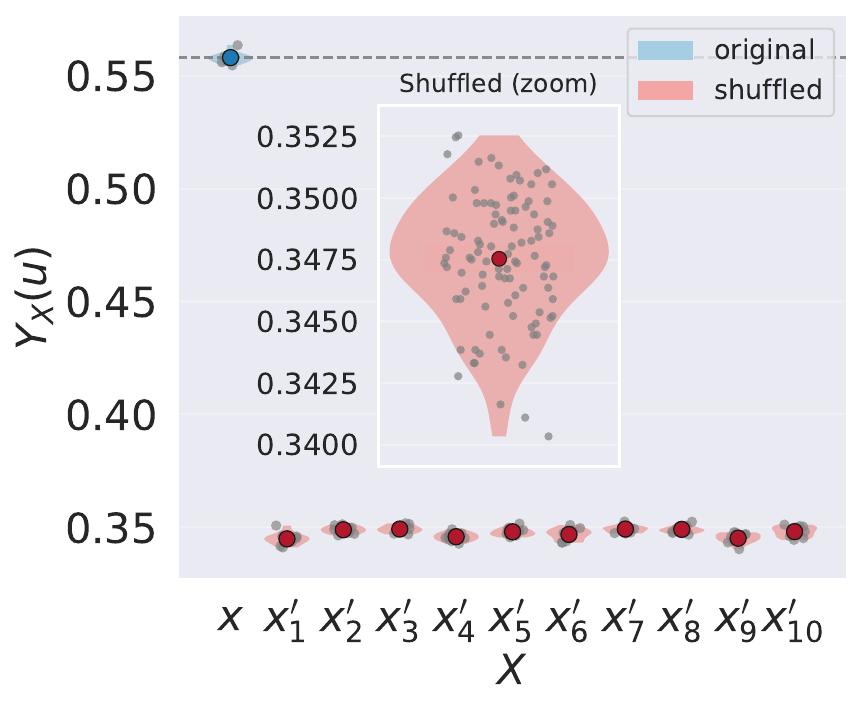}\label{fig:tgn}}
    \subfigure[AP, \texttt{TGN}]{\includegraphics[width=0.28\columnwidth]{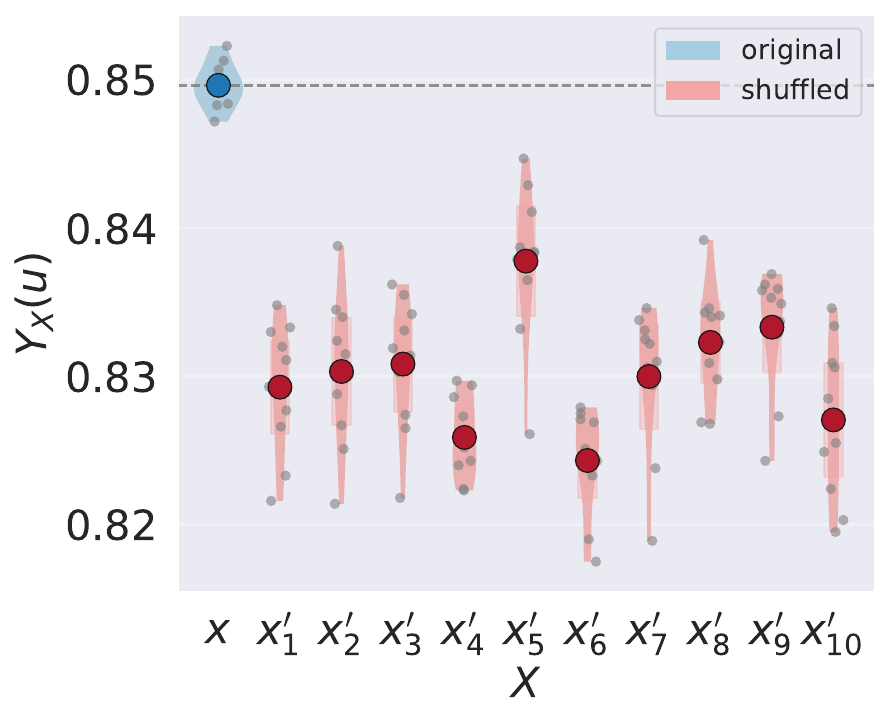}\label{fig:tgn}}
    \subfigure[ AP, \texttt{JODIE}]{\includegraphics[width=0.28\columnwidth]{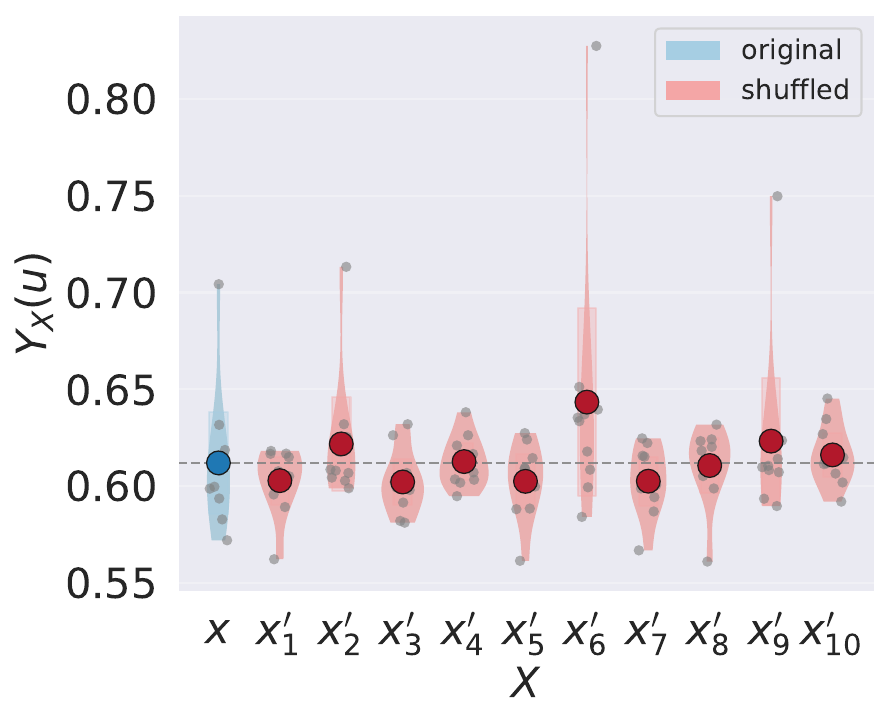}\label{fig:jodie}} 
    \caption{Violin plots of model performance under timestamp shuffling.}
    \label{fig:shuffle_tlp}
\end{figure}

\vspace{-25pt}
\section{Conclusion}
\label{sec:conclude}
In this work we presented an algorithm to generate causal event sequences, and a metric to measure the
distance between parametric models generating them. We then extended the mechanism to temporal
interaction graphs by proposing a construction in which node features induce a causal graph over edge
events, yielding causal temporal interaction graphs (CTIGs) with known ground-truth structure. Using
this setup, we instantiated counterfactual evaluation in two complementary ways: by testing models
across CTIGs generated from different causal models, and by applying timestamp shuffling as a
stochastic temporal distortion whose severity can be quantified by the proposed distance. In both
settings, we observed  degradation in predictive performance under causal distortion; in
particular, TGN exhibits clear sensitivity to the induced distortions whereas JODIE remains largely flat. 
Our evaluation framework offers a controlled testbed for developing and assessing TLP models that can exploit underlying causal structure rather than superficial correlations. Furthermore, the same methodology can be extended to other predictive domains where the performance of a model under causal shifts is of interest.




\bibliography{Causality}

@article{gong_active_2023,
  title = {Active Causal Structure Learning in Continuous Time},
  author = {Gong, Tianwei and Gerstenberg, Tobias and Mayrhofer, Ralf and Bramley, Neil R.},
  year = {2023},
  month = feb,
  journal = {Cognitive Psychology},
  volume = {140},
  pages = {101542},
  issn = {0010-0285},
  doi = {10.1016/j.cogpsych.2022.101542},
  urldate = {2024-10-21}
}

@book{haenggi2012stochastic,
  title={Stochastic geometry for wireless networks},
  author={Haenggi, Martin},
  year={2012},
  publisher={Cambridge University Press}
}

@article{kim2011granger,
  title={A Granger causality measure for point process models of ensemble neural spiking activity},
  author={Kim, Sanggyun and Putrino, David and Ghosh, Soumya and Brown, Emery N},
  journal={PLoS computational biology},
  volume={7},
  number={3},
  pages={e1001110},
  year={2011},
  publisher={Public Library of Science San Francisco, USA}
}

@book{pearl2009causality,
  title={Causality},
  author={Pearl, Judea},
  year={2009},
  publisher={Cambridge university press}
}

@article{noorbakhsh2022counterfactual,
  title={Counterfactual temporal point processes},
  author={Noorbakhsh, Kimia and Rodriguez, Manuel},
  journal={Advances in Neural Information Processing Systems},
  volume={35},
  pages={24810--24823},
  year={2022}
}

@inproceedings{cupperscausal,
  title={Causal Discovery from Event Sequences by Local Cause-Effect Attribution},
  author={C{\"u}ppers, Joscha and Xu, Sascha and Musa, Ahmed and Vreeken, Jilles},
  booktitle={The Thirty-eighth Annual Conference on Neural Information Processing Systems},
  year= {2024}
}

@article{barabasi2014network,
  title={Network science book},
  author={Barab{\'a}si, Albert-L{\'a}szl{\'o}},
  journal={Network Science},
  volume={625},
  year={2014}
}

@article{abbe2018community,
  title={Community detection and stochastic block models: recent developments},
  author={Abbe, Emmanuel},
  journal={Journal of Machine Learning Research},
  volume={18},
  number={177},
  pages={1--86},
  year={2018}
}

@article{hawkes1971spectra,
  title={Spectra of some self-exciting and mutually exciting point processes},
  author={Hawkes, Alan G},
  journal={Biometrika},
  volume={58},
  number={1},
  pages={83--90},
  year={1971},
  publisher={Oxford University Press}
}

@inproceedings{jalaldoust2022causal,
  title={Causal discovery in Hawkes processes by minimum description length},
  author={Jalaldoust, Amirkasra and Hlav{\'a}{\v{c}}kov{\'a}-Schindler, Kate{\v{r}}ina and Plant, Claudia},
  booktitle={Proceedings of the AAAI Conference on Artificial Intelligence},
  volume={36},
  number={6},
  pages={6978--6987},
  year={2022}
}

@article{zhao2021event,
  title={Event prediction in the big data era: A systematic survey},
  author={Zhao, Liang},
  journal={ACM Computing Surveys (CSUR)},
  volume={54},
  number={5},
  pages={1--37},
  year={2021},
  publisher={ACM New York, NY, USA}
}

@article{erdos1961evolution,
  title={On the evolution of random graphs},
  author={Erdos, Paul},
  journal={Bulletin of the Institute of International Statistics},
  volume={38},
  pages={343--347},
  year={1961}
}

@inproceedings{kumar2019predicting,
  title={Predicting dynamic embedding trajectory in temporal interaction networks},
  author={Kumar, Srijan and Zhang, Xikun and Leskovec, Jure},
  booktitle={Proceedings of the 25th ACM SIGKDD international conference on knowledge discovery \& data mining},
  pages={1269--1278},
  year={2019}
}

@article{ur2025primer,
  title={A Primer on Temporal Graph Learning},
  author={Rahman, Aniq Ur and Elhag, Ahmed A and Coon, Justin P},
  journal={ACM Computing Surveys},
  volume={58},
  number={5},
  pages={1--28},
  year={2025},
  publisher={ACM New York, NY}
}

@article{mansouri2025learning,
  title={Learning from positive and unlabeled examples-Finite size sample bounds},
  author={Mansouri, Farnam and Ben-David, Shai},
  journal={arXiv preprint arXiv:2507.07354},
  year={2025}
}

@article{yu2023towards,
  title={Towards better dynamic graph learning: New architecture and unified library},
  author={Yu, Le and Sun, Leilei and Du, Bowen and Lv, Weifeng},
  journal={Advances in Neural Information Processing Systems},
  volume={36},
  pages={67686--67700},
  year={2023}
}

@article{rossi_temporal_2020,
	title = {Temporal graph networks for deep learning on dynamic graphs},
	journal = {arXiv preprint arXiv:2006.10637},
	author = {Rossi, Emanuele and Chamberlain, Ben and Frasca, Fabrizio and Eynard, Davide and Monti, Federico and Bronstein, Michael},
	year = {2020},
}

@inproceedings{
rahman2025rethinking,
title={Rethinking Evaluation for Temporal Link Prediction through Counterfactual Analysis},
author={Aniq Ur Rahman and Alexander Modell and Justin Coon},
booktitle={I Can't Believe It's Not Better: Challenges in Applied Deep Learning},
year={2025},
url={https://openreview.net/forum?id=TKydQh6koc}
}

\appendix

\section{Preliminaries}
\label{sec:prelims}

\begin{definition}[\textit{Point Process}, {\citep[\textsection~2.2]{haenggi2012stochastic}}]
    A point process is a countable random collection of points that reside in some measure space, usually the Euclidean space $\mathbb{R}^d$. The associated $\sigma$-algebra consists of the Borel sets $\gB^d$, and the measure is the Lebesgue measure.
\end{definition}
In simple terms, a point process is a countable random set $\Phi = \left\{ x_1, x_2, \cdots \right\} \subseteq \mathbb{R}^d$ consisting of random variables $x_i \in \mathbb{R}^d$ as its elements. For a set $B \subseteq \mathbb{R}^d$, the cardinality of $\Phi \cap B$ is denoted as $N(B) \in \mathbb{N}$, and $N(\cdot)$ is called a counting measure.

\begin{definition}[\textit{1-D Poisson Point Process}, {\citep[\textsection~2.4.1]{haenggi2012stochastic}}]
The one-dimensional Poisson point process (PPP) with parameter $\lambda \in \mathbb{R}^+$ is a point process in $\mathbb{R}$ such that
\begin{enumerate}
    \item for every bounded interval $[a, b)$, the number of points in the interval $N([a,b)]) \in \mathbb{N}$ has a Poisson distribution with mean $\lambda(b-a)$, i.e., $P(N([a,b)])=k) = e^{-\lambda(b-a)} \frac{(\lambda(b-a))^k}{k!},$ and
    \item if $[a_1, b_1), [a_2, b_2), \cdots,  [a_m, b_m)$ are disjoint bounded intervals, then the number of points in those intervals $N([a_1, b_1)), \allowbreak \cdots,  N([a_m, b_m))$ are independent random variables.
\end{enumerate}
\label{def:PPP}
\end{definition}

We denote the samples drawn from a PPP with parameter $\lambda$ as $\Phi(\lambda) \subset \mathbb{R}$, and $(\Phi(\lambda), <)$ satisfies the properties of a \textit{total order}, i.e., $\Phi(\lambda) = \{  t_1, t_2, \cdots, t_m  \},$ such that $t_i < t_{i+1}, \forall i \in [m]$ for some $m \in \mathbb{N}^+.$

\begin{definition}[\textit{Causal model}, {\citep[\textsection~2.2, \textsection~7.1]{pearl2009causality}} ]
    Let $\gU = \{ U_1, \cdots, \allowbreak U_n \}$ be a set of background variables determined by factors outside the model, and let $\gV=\{ X_1, \cdots, X_n\}$ be a set of endogenous variables that are determined by variables in the model, i.e., variables in $\gU \cup \gV$. A causal model is a pair $\mathfrak{M} = ( \graph, \Theta_{\graph} )$ consisting of a causal graph\footnote{Causal graph is a directed graph where an edge from node $X$ to $Y$ indicates that $X$ causes $Y$.} $\graph$ and a set of parameters $\Theta_{\graph}$ compatible\footnote{The elements in $\Theta_{\graph}$ specify the functions $f_i: U_i \cup \pa_i \mapsto X_i, \forall i \in [n]$.} with $\graph$. The parameters $\Theta_{\graph}$ assign a function $x_i = f_i(\pa_i(t), u_i)$ to each variable $X_i \in \gV$ and a probability measure $\mathbb{P}(u_i)$ to each $u_i$, where $\pa_i$ are the parents of $X_i$ and where each $U_i \in \gU$ is a randomly  distributed according to $\mathbb{P}(u_i)$. 
    \label{def:CM}
\end{definition}

\begin{figure}[h!]
    \centering
    \includegraphics[width=0.35\linewidth]{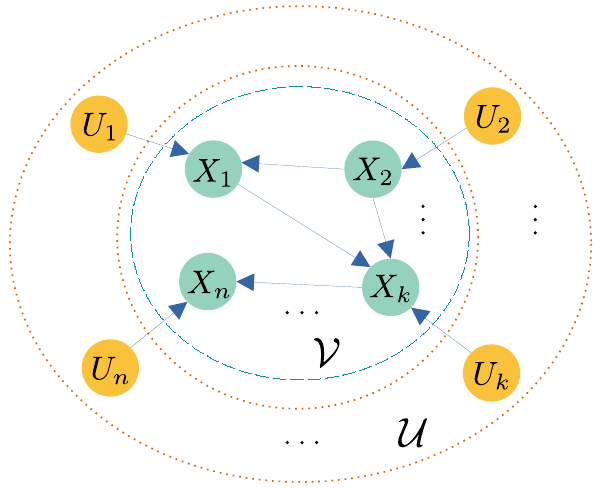}
    \caption{Causal model.}
    \label{fig:cgm}
\end{figure}

\begin{definition}[\textit{Monotonicity}, {\citep[\textsection~9.2]{pearl2009causality}} ]
    A variable $Y$ is said to be monotonic relative to variable $X$ in a causal model if and only if
    \begin{align*}
        \mathbb{I}\{ Y = 0 \mid do(X = 1) \} \cdot \mathbb{I}\{ Y = 1 \mid do(X = 0) \} = 0.
    \end{align*}
    \label{def:mono}
\end{definition}

\begin{definition}[\textit{Exogeneity}, {\citep[\textsection~7.4]{pearl2009causality}} ]
    In a causal model $( \graph, \Theta_{\graph} )$ a variable $X$ is exogenous relative to $Y$ if $X$ and $Y$ have no common ancestor in $\graph$.
    \label{def:exo}
\end{definition}

\section{Related Works}
\label{app:relworks}

The following table highlights key limitations in existing literature and explains how our proposed causal model resolves them.  
\begin{table*}[h!]
\small
\centering
\caption{Limitations in the literature and their resolution through our proposed causal model.}
\label{tab:limitations}
\begin{tabular}{p{0.5cm}  p{6cm} p{6cm}}
\toprule
\textit{} & \textit{Limitation} & \textit{Resolution} \\ 
\midrule
\multicolumn{3}{l}{\citet{noorbakhsh2022counterfactual}:}\\
(\textit{L1}) & The thinning procedure of a variable is independent of other variables.  & We define the acceptance probability as a function of other variables, through structural equations.\\  
\midrule
\multicolumn{3}{l}{\citet{cupperscausal}:}\\
(\textit{L2}) & An event of type $i$ at time $t$ can at most cause a single event of type $j$ at some time $t' >t$. & The influence of an event lasts for a finite duration within which it can cause an event $i$ multiple times.\\
(\textit{L3}) & An event is the result of a single cause, and a more realistic scenario of multiple causes working towards an effect is unexplored.  & In the SEM \eqref{sem} we can see that multiple events jointly cause an event. \\
\midrule 
\multicolumn{3}{l}{\citet{cupperscausal, noorbakhsh2022counterfactual, jalaldoust2022causal}:}\\
(\textit{L4}) & The cause-effect relationships are strictly positive, i.e., a cause can only result in the occurrence of the effect. & In this work, we look at both positive and negative cause-effect relationships, i.e., excitatory and inhibitory. \\
\bottomrule
\end{tabular}
\end{table*}

\section{Properties of the Causal Model}
\label{app:properties}

\begin{proposition}
    The causal model $( \graph, \Theta_{\graph} )$ is \underline{semi-Markovian} for all $\rmA \in \{0,1\}^{n \times n}$.
    \label{prop:semimarkov}
\end{proposition}
\begin{proof}
    The causal graph $\graph \subset \{ \gU \times \gV \} \cup \{ \gV' \times \gV \}$ consists of  two unidrectional bipartite subgraphs, one from $\gU$ to $\gV$, and another from $\gV'$ to $\gV$. Starting from any node, there is no way of reaching nodes in $\gU$ or $\gV'$. Moreover, starting from a node in $\gV$ there is no way to return to it, as there are no edges directed away from the nodes in $\gV$. Therefore, it is impossible to form a directed cycle in $\graph$ allowing us to conclude that $\graph$ is \textit{acyclic}, and thus the causal model $( \graph, \Theta_{\graph} )$ is semi-Markovian \citep[\textsection~3.2]{pearl2009causality}.
\end{proof}

\begin{proposition}
    The causal model $( \graph, \Theta_{\graph} )$  is \underline{not Markovian} for all $\rmA \in \{0,1\}^{n \times n}$.
    \label{prop:nonmarkov}
\end{proposition}
\begin{proof}
    In the proposed model $( \graph, \Theta_{\graph} )$ the set $\gU \cup \gV'$ represents the background variables, while $\gV$ constitutes the endogenous variables. For a model to be Markovian, it should be semi-Markovian, and in addition the background variables should be \textit{independent} \citep[\textsection~3.2]{pearl2009causality}. Therefore, the semi-Markovian model $( \graph, \Theta_{\graph} )$ is also Markovian if
    \begin{enumerate}[label=(C\arabic*)]
        \item $U_i(t) \indep  U_j(t), \, \forall i, j \in [n], \, i \neq j$, 
        \item $X_i'(t) \indep U_j(t) , \, \forall i, j \in [n]$,
        \item $X_i'(t) \indep  X_j'(t), \, \forall i, j \in [n], \, i \neq j$. 
    \end{enumerate}

    \par 
    \textit{Verifying (C1):} We express $ U_i(t) \indep U_j(t) $ as $$ \mathbb{I}\{ t \in \Phi(\lambda_i) \}  \indep \mathbb{I}\{ t \in \Phi(\lambda_j) \}, $$ which is equivalent to verifying whether $ \Phi(\lambda_i) \indep \Phi(\lambda_j) $. Since the two PPPs are independent, (C1) holds \textit{true}. 

    \par 
    \textit{Verifying (C2):} We express $X_i'(t)$ and $U_j(t)$ in terms of the PPPs:
    \begin{align*}
        &\textstyle X_i'(t) = \mathbb{I}\left\{ \sum_{\bar{t} \in (t - \bar{\tau}, t)} X_i(\bar{t}) > 0 \right\} \\
        &= \mathbb{I}\left\{ \sum_{\bar{t} \in (t - \bar{\tau}, t)} f\left( \bigcup_{j \in \pa_i(\bar{t})} \mathbb{I}\{ |\Phi(\lambda_j) \cap (\bar{t}-\bar{\tau}, \bar{t})| > 0 \} , U_i(\bar{t} ) \right) > 0 \right\}.\\
        &\textstyle U_j(t) = \mathbb{I}\{ t \in \Phi(\lambda_j) \}  = \mathbb{I}\{\lim_{\epsilon \downarrow 0}| \Phi(\lambda_j) \cap [t, t+ \epsilon) | > 0\}.
    \end{align*}
    The value of $X_i'(t)$ depends upon $\Phi(\lambda_k) \cap (t - 2 \bar{\tau}, t), \, \forall k \in \pa_i$ and $\Phi(\lambda_i) \cap (t - \bar{\tau}, t)$. Similarly, the value of $U_j(t)$ depends upon $\Phi(\lambda_j) \cap [t, t+ \epsilon)$ for some $\epsilon \downarrow 0$. 
    Since there is no overlap between $(t - 2 \bar{\tau}, t)$ and $\lim_{\epsilon \downarrow 0} [t, t+ \epsilon)$, (C2) holds \textit{true}.

    \par 
    \textit{Verifying (C3):} From the previous paragraph, we can say that $X_i'(t)$ depends on $\Phi(\lambda_k) \cap (t - 2 \bar{\tau}, t), \, \forall k \in \pa_i(t)$ and $\Phi(\lambda_i) \cap (t - \bar{\tau}, t)$, and $X_j'(t)$ depends on $\Phi(\lambda_m) \cap (t - 2 \bar{\tau}, t), \, \forall m \in \pa_j(t)$ and $\Phi(\lambda_j) \cap (t - \bar{\tau}, t)$. 

    Intuitively, this means that each $X_i'(t)$ is influenced by the recent activity of its parent events and its own recent occurrences within their respective temporal windows. For $X_i'(t)$ and $X_j'(t)$ to be independent, their corresponding sets of influencing events must not overlap at any time $t$. In other words, their dependency domains must remain disjoint over time. 
    
    Therefore, $X_i'(t) \indep X_j'(t)$ only if
    \begin{align*}
        &\{ \pa_i(t) \cup \{ i : X_i'(t) = 1 \} \} \cap \{ \pa_j(t) \cup \{ j : X_j'(t) = 1 \} \} = \emptyset, \quad \forall i \neq j, \forall t \in \mathbb{R}^+ \\  
        &\implies \{ \pa_i \cup \{i\} \} \cap \{ \pa_j \cup \{j\} \} = \emptyset, \quad \forall i \neq j \\
        &\implies \pa_i \subseteq \{ i \}, \forall i \in [n]. 
    \end{align*}
    
    The implication follows under the assumption that, for some $t$, the time-varying parent sets $\pa_i(t)$ and $\pa_j(t)$ coincide with their structural counterparts $\pa_i$ and $\pa_j$, and that the indicators $X_i'(t)=1$ and $X_j'(t)=1$ can be simultaneously satisfied. This condition effectively requires that the sets of potential causal influencers  remain completely non-overlapping. 
    However, such disjointness is highly restrictive and can only occur for $\rmA = \rmI$. Therefore, we conclude that (C3) is \textit{not true} for all realizations of $\rmA$.

    Since (C3) does not always hold, we conclude that the causal model $( \graph, \Theta_{\graph} )$  is not Markovian for all $\rmA$.
\end{proof}

\begin{corollary}
    The causal model $( \graph, \Theta_{\graph} )$ is \underline{Markovian} if $\rmA = \rmI$.
    \label{cor:markov}
\end{corollary}

\begin{proposition}
    $X_i(t)$ is \underline{monotonic} relative to $X_j'(t)$ iff $\Theta_{i,j} \in [0,1]$.
    \label{prop:mono}
\end{proposition}
\begin{proof}
    From def.~\ref{def:mono}, we can say that $X_i(t)$ is monotonic relative to $X_j'(t)$ if and only if
    \begin{align*}
        \mathbb{I}\!\left\{ X_i(t) = 0 \mid do\!\left(X_j'(t)=1 \right) \right\}
        \mathbb{I}\!\left\{ X_i(t) = 1 \mid do\!\left(X_j'(t)=0 \right) \right\} = 0.
    \end{align*}
    The above equation must hold for all subsets $\gC_i \subseteq \pa_i \setminus \{j\}$, i.e.,
    \begin{align}
        \mathbb{I}\!\left\{ \Theta_{i,j} + \sum_{k \in \gC_i} \Theta_{i,k} < 0 \right\}
        \mathbb{I}\!\left\{ \sum_{k \in \gC_i} \Theta_{i,k} \ge 0 \right\}
        = 0, \quad \forall \gC_i \subseteq \pa_i \setminus \{j\}.
        \label{monotonic}
    \end{align}

    \textit{Case 1:} $\sum_{k \in \gC_i} \Theta_{i,k} \ge 0$.  
    To ensure that \eqref{monotonic} holds, we require
    \begin{align*}
        \Theta_{i,j} + \sum_{k \in \gC_i} \Theta_{i,k} \ge 0.
    \end{align*}
    Taking the intersection over all such $\gC_i$, we obtain
    \begin{align*}
        \Theta_{i,j}
        &\in \bigcap_{\gC_i : \sum_{k \in \gC_i} \Theta_{i,k} \ge 0}
        \left[ -\sum_{k \in \gC_i} \Theta_{i,k},\, 1 \right] \\
        &= \left[
            \max\!\left\{
                -\sum_{k \in \gC_i} \Theta_{i,k}
                : \sum_{k \in \gC_i} \Theta_{i,k} \ge 0,\,
                \forall \gC_i \subseteq \pa_i \setminus \{j\}
            \right\},
            1
        \right].
    \end{align*}
    For $\gC_i = \emptyset$ we have $\sum_{k \in \gC_i} \Theta_{i,k} = 0$, yielding
    \begin{align*}
        \Theta_{i,j} \in [0,1].
    \end{align*}

    \textit{Case 2:} $\sum_{k \in \gC_i} \Theta_{i,k} < 0$.  
    In this case, $\mathbb{I}\left\{\sum_{k \in \gC_i} \Theta_{i,k} \ge 0 \right\} = 0$, so
    \eqref{monotonic} is automatically satisfied for all $\Theta_{i,j}$.  
    Hence, these subsets impose no additional restriction beyond the assumed domain
    $\Theta_{i,j} \in [-1,1]$.

    Intersecting the results from both cases gives $\Theta_{i,j} \in [0,1]$, which completes the proof.
\end{proof}

\begin{corollary}
    If for some $i \in [n]$, $\Theta_{i,j} \in [0,1], \forall j \in \pa_i$ then the variable $X_i(t)$ is monotonic relative to all $X_j'(t), \forall j \in \pa_i$. However, the structural equation is reduced to $x_i(t) = u_i(t)$ and the causal impact of $X_j'(t)$ on $X_i(t)$ vanishes\footnote{To rule out trivial feasibility, we verify that $\lvert S \rvert \neq \lvert \Phi \rvert$ for all datasets (equality $\lvert S \rvert = \lvert \Phi \rvert$ would indicate degeneracy); the verification scripts will be included in the public code release.} $\forall j \in \pa_i$.
    \label{cor:vanish}
\end{corollary}

\begin{proposition}
    $X_j'(t)$ is \underline{exogenous} relative to $X_i(t) \, \forall i,j \in [n]$.
    \label{prop:exo}
\end{proposition}
\begin{proof}
    Firstly, for all $\rmA \in \{0,1\}^{n \times n}$ the subgraph with edges $\gV' \times \gV$ is a unidirectional bipartite graph. Secondly, $X_j'(t) \in \gV'$ has no ancestors. Therefore, $X_i(t)$  and $X_j'(t)$ have no common ancestor $\forall i,j \in [n]$ which concludes the proof (see def.~\ref{def:exo}).
\end{proof}

\begin{theorem}[Identifiability]
\label{thm:identify}
    If for some $j \in \pa_i$ $X_j'(t)$  exogenous relative to $X_i(t)$, and $X_i(t)$ is monotonic relative to $X_j'(t)$, then the probability of necessity and sufficiency $(\mathsf{PNS})$ is \underline{identifiable}, and determined through:
    \begin{align*}
        \mathsf{PNS}_{i,j} = \mathbb{P}\left(X_i = 1 \mid X_j'= 1 \right) -  \mathbb{P}\left(X_i = 1 \mid X_j'= 0 \right).
    \end{align*}
\end{theorem}
\begin{proof}
    For detailed proof, please see \citep[\textsection~9.2.3]{pearl2009causality}.
\end{proof}

\end{document}